\documentclass[12pt]{article}

\usepackage[utf8]{inputenc}
\usepackage{amsmath,amsthm,amssymb,fullpage,graphicx}
\usepackage{algorithm,algorithmicx,algpseudocode,verbatim,color}
\usepackage[colorlinks=true,linkcolor=blue,citecolor=blue,linktocpage=true]{hyperref}

\newcommand{\ex}[2]{{\ifx&#1& \mathbb{E} \else \underset{#1}{\mathbb{E}} \fi \left[#2\right]}}
\newcommand{\pr}[2]{{\ifx&#1& \mathbb{P} \else \underset{#1}{\mathbb{P}} \fi \left[#2\right]}}
\newcommand{\exc}[3]{{\ifx&#1& \mathbb{E} \else \underset{#1}{\mathbb{E}} \fi \left[ #2 \middle| #3 \right]}}
\newcommand{\prc}[3]{{\ifx&#1& \mathbb{P} \else \underset{#1}{\mathbb{P}} \fi \left[ #2 \middle| #3 \right]}}
\newcommand{\var}[2]{{\ifx&#1& \mathbb{V} \else \underset{#1}{\mathbb{V}} \fi \left[#2\right]}}
\newcommand{\dr}[3]{\mathrm{D}_{#1}\left(#2\middle\|#3\right)}

\newcommand{\sign}{\mathsf{sign}}
\newcommand{\tvd}[2]{d_{\text{TV}}\left(#1,#2\right)}

\newcommand{\nope}[1]{}

\newtheorem{theorem}{Theorem}
\newtheorem{lemma}[theorem]{Lemma}
\newtheorem{definition}[theorem]{Definition}
\newtheorem{corollary}[theorem]{Corollary}
\newtheorem{proposition}[theorem]{Proposition}
\newtheorem{remark}[theorem]{Remark}
\numberwithin{theorem}{section}

\usepackage[style=alphabetic,backend=bibtex,maxalphanames=10,maxbibnames=20,maxcitenames=10,giveninits=true,doi=false,url=true,backref=true]{biblatex}
\newcommand*{\citetal}[1]{\AtNextCite{\AtEachCitekey{\defcounter{maxnames}{2}}}\textcite{#1}}
\newcommand*{\citet}[1]{\AtNextCite{\AtEachCitekey{\defcounter{maxnames}{999}}}\textcite{#1}}
\newcommand*{\citep}[1]{\cite{#1}}

\title{Privacy Auditing with One (1) Training Run}
 
\author{
    \href{http://www.thomas-steinke.net/}{\color{black}{Thomas Steinke}}\thanks{Google. \texttt{\{steinke,srxzr,jagielski\}@google.com}. Reverse alphabetical author order.}%
    \and 
    \href{https://people.cs.umass.edu/~milad/}{\color{black}{Milad Nasr}}\footnotemark[1]%
    \and
    \href{https://jagielski.github.io/}{\color{black}{Matthew Jagielski}}\footnotemark[1]%
}

\date{}%

\begin{document}
\maketitle
\begin{abstract}
    We propose a scheme for auditing differentially private machine learning systems with a single training run. This exploits the parallelism of being able to add or remove multiple training examples independently. We analyze this using the connection between differential privacy and statistical generalization, which avoids the cost of group privacy. Our auditing scheme requires minimal assumptions about the algorithm and can be applied in the black-box or white-box setting.
\end{abstract}

\section{Introduction}
Differential privacy (DP) \cite{dwork2006calibrating} provides a quantifiable privacy guarantee by ensuring that no person's data significantly affects the probability of any outcome. 
Formally, a randomized algorithm $M$ satisfies $(\varepsilon,\delta)$-DP if, for any pair of inputs $x,x'$ differing only by the addition or removal of one person's data and any measurable $S$, we have 
\begin{equation}
    \pr{}{M(x) \in S} \le e^\varepsilon \cdot \pr{}{M(x') \in S} + \delta.\label{eq:dp}
\end{equation}
A DP algorithm is accompanied by a mathematical proof giving an \emph{upper bound} on the privacy parameters $\varepsilon$ and $\delta$.
In contrast, a \emph{privacy audit} provides an empirical \emph{lower bound} on the privacy parameters.
Privacy audits allow us to assess the tightness of the mathematical analysis \cite{jagielski2020auditing,nasr2023tight} or, if the lower and upper bounds are contradictory, to detect errors in the analysis or in the algorithm's implementation \cite{tramer2022debugging}.

Typically, privacy audits obtain a lower bound on the privacy parameters directly from the DP definition \eqref{eq:dp}.
That is, we construct a pair of inputs $x,x'$ and a set of outcomes $S$ and we estimate the probabilities $\pr{}{M(x) \in S}$ and $\pr{}{M(x') \in S}$.
However, estimating these probabilities requires running the algorithm $M$ hundreds of times.
This approach to privacy auditing is computationally expensive, which raises the question
\begin{quote}
    \emph{Can we perform privacy auditing using a single run of the algorithm M?}
\end{quote}
This is the question we address in our work.

\subsection{Our Contributions}

\textbf{Our approach (\S\ref{sec:procedure}):} 
The DP definition \eqref{eq:dp} considers adding or removing a single person's data to or from the dataset.
We consider multiple people's data and the dataset independently includes or excludes each person's data point. Our analysis exploits the parallelism of multiple independent data points in a single run of the algorithm in lieu of multiple independent runs.

Our auditing procedure operates as follows. We identify $m$ data points (i.e., training examples or ``canaries'') to either include or exclude and we flip $m$ independent unbiased coins to decide which of them to include or exclude. We then run the algorithm on the randomly selected dataset. Based on the output of the algorithm, the auditor ``guesses'' whether or not each data point was included or excluded (or it can abstain from guessing for some data points). We obtain a lower bound on the privacy parameters from the fraction of guesses that were correct. 

Intuitively, if the algorithm is $(\varepsilon,0)$-DP, then the auditor can correctly guess each inclusion/exclusion coin flip with probability at most $\frac{e^\varepsilon}{e^\varepsilon+1}$. Thus DP implies a high-probability upper bound on the fraction of correct guesses and, conversely, a large fraction of correct guesses implies a high-probability lower bound on the privacy parameters.

\textbf{Our analysis (\S\ref{sec:theory}):}
Na\"ively, analyzing the addition or removal of multiple data elements would rely on group privacy; but this does not exploit the fact that the data items were included or excluded independently.
Instead, we leverage the connection between DP and generalization \cite{dwork2015preserving,dwork2015generalization,bassily2016algorithmic,rogers2016max,jung2019new,steinke2020reasoning}.
Our main theoretical contribution is an improved analysis of this connection that is tailored to yield nearly tight bounds in our setting.

Informally, if we run a DP algorithm on i.i.d.~samples from some distribution, then, conditioned on the output of the algorithm, the samples are still ``close'' to being i.i.d.~samples from that distribution.
There is some technicality in making this precise, but, roughly speaking, we show that including or excluding $m$ data points independently for one run is essentially as good as having $m$ independent runs (as long as $\delta$ is small).

\textbf{Our results (\S\ref{sec:experiments}):} We implement our new auditing framework to audit DP-SGD training on a WideResNet model, trained on the CIFAR10 dataset across multiple configurations. Our approach successfully achieves an empirical lower bound of $\varepsilon \ge 1.8$, compared to a theoretical upper bound of $\varepsilon \le 4$ in the white-box setting. The $m$ examples we insert for auditing (known in the literature as ``canaries'') do not significantly impact the accuracy of the final model (less than a $5\%$ decrease in accuracy) and our procedure only requires a single end-to-end training run. Such results were previously unattainable in the setting where only one model could be trained.

\section{Our Auditing Procedure}\label{sec:procedure}

\newcommand{\tin}{{\text{IN}}}
\newcommand{\tout}{{\text{OUT}}}
\newcommand{\xin}{x_{\text{IN}}}
\newcommand{\xdata}{x_{\text{FIXED}}}
\newcommand{\xout}{x_{\text{OUT}}}
\newcommand{\nin}{{n_{\text{IN}}}}
\newcommand{\nout}{{n_{\text{OUT}}}}
\newcommand{\scin}{s_{\text{IN}}}
\newcommand{\sout}{s_{\text{OUT}}}
\begin{algorithm}[h]
    \caption{Auditor with One Training Run}\label{alg:generic-audit}
    \begin{algorithmic}[1]
        \State \textbf{Data:} $x \in \mathcal{X}^{n}$ consisting of $m$ auditing examples (a.k.a.~canaries) and $n-m$ non-auditing examples.
        \State \textbf{Parameters:} Algorithm to audit $\mathcal{A}$, number of examples to randomize $m$, number of positive $k_+$ and negative $k_-$ guesses.
        \State For $i \in [m]$ sample $S_i \in \{-1,+1\}$ independently with $\ex{}{S_i}=0$. Set $S_i=1$ for all $i \in [n] \setminus [m]$.
        \State Partition $x$ into $\xin \in \mathcal{X}^\nin$ and $\xout \in \mathcal{X}^\nout$ according to $S$, where $\nin+\nout=n$. Namely, if $S_i=1$, then $x_i$ is in $\xin$; and, if $S_i=-1$, then $x_i$ is in $\xout$.
        \State Run $\mathcal{A}$ on input $\xin$ with appropriate parameters, outputting $w$.
        \State Compute the vector of scores $Y = \left( \textsc{Score}(x_i,w) : i \in [m] \right) \in \mathbb{R}^m$.
        \State Sort the scores $Y$. Let $T \in \{-1,0,+1\}^m$ be $+1$ for the largest $k_+$ scores and $-1$ for the smallest $k_-$ scores.\\(I.e., $T \in \{-1,0,+1\}^m$ maximizes $\sum_i^m T_i \cdot Y_i$ subject to $\sum_i^m |T_i| = k_++k_-$ and $\sum_i^m T_i = k_+-k_-$.)
        \State \textbf{Return:} The vector $S \in \{-1,+1\}^m$ indicating the true selection and the guesses $T \in \{-1,0,+1\}^m$.
    \end{algorithmic}
\end{algorithm}

We now present our auditing procedure in Algorithm \ref{alg:generic-audit}. %
We independently include each of the first $m$ examples with 50\% probability and exclude it otherwise.\footnote{Alternatively, we could also consider a different probability of inclusion; our theoretical results can handle this (see Proposition \ref{prop:approx}). However, this seems unlikely to be useful, as it intuitively lowers the signal-to-noise ratio. Another alternative is to non-independently choose which points to include to ensure $\xin$ has a fixed size; see Appendix \ref{sec:replace_dp}.}
Our approach is applicable to both white-box auditing in the sense that the adversary has access to all intermediate values of the model weights and black-box auditing in the sense that the adversary only sees the final model weights (or can only query the final model). 
In both cases we compute a ``score'' for each example and ``guess'' whether the example is included or excluded based on these scores.
Specifically, we guess that the examples with the $k_+$ highest scores are included and the examples with the $k_-$ lowest scores are excluded, and we abstain from guessing for the remaining $m-k_+-k_-$ auditing examples; the setting of these parameters will depend on the application.

Note that we only randomize the first $m$ examples $x_1, \cdots, x_m$ (which we refer to as ``auditing examples'' or ``canaries''); the last $n-m$ examples $x_{m+1}, \cdots, x_n$ are always included and, thus, we do not make any guesses about them. To get the strongest auditing results we would set $m=n$, but we usually want to set $m<n$. For example, computing the score of all $n$ examples may be computationally prohibitive, so we only compute the scores of $m$ examples. Also we may wish to artificially construct $m$ examples to be easy to identify (i.e., canaries), but still include $n-m$ ``real'' examples to ensure that $\mathcal{A}$ still produces a useful model. (I.e., having more training examples improves the performance of the model.) 

Intuitively, the vector of scores $Y$ should be correlated with the true selection $S$, but too strong a correlation would violate DP. This is the basis of our audit.
Specifically, the auditor computes $T$ from $Y$ which is a ``guess'' at $S$. By the postprocessing property of DP, the guesses $T$ are a differentially private function of the true $S$, which means that they cannot be too accurate. 

To obtain a lower bound on the DP parameters, in Section \ref{sec:theory}, we show that DP implies a high-probability upper bound on the number of correct guesses $W := \sum_i^m \max\{ 0 , T_i \cdot S_i \}$. The observed value of $W$ then yields a high-probability lower bound on the DP parameters. To be more precise, we have the following guarantee.

\begin{theorem}[Informal version of Theorem \ref{thm:main}]\label{thm:main-informal}
    Let $(S,T) \in \{-1,+1\}^m \times \{-1,0,+1\}^m$ be the output of Algorithm \ref{alg:generic-audit}. Assume the algorithm to audit $\mathcal{A}$ satisfies $(\varepsilon,\delta)$-DP.
    Let $r:=k_++k_-=\|T\|_1$ be the number of guesses.
    Then, for all $v \in \mathbb{R}$,
    \begin{equation}\pr{S \gets \{-1,+1\}^m \atop T \gets M(S)}{\sum_i^m \max\{0, T_i \cdot S_i\} \ge  v} \le \pr{\check{W} \gets \mathsf{Binomial}\left(r,\frac{e^\varepsilon}{e^\varepsilon+1}\right)}{\check{W} \ge v} +  O(\delta).\label{eq:main-inf}\end{equation}
\end{theorem}

If we ignore $\delta$ for the moment, Theorem \ref{thm:main-informal} says that the number of correct guesses is stochastically dominated by $\mathsf{Binomial}\left(r,\frac{e^\varepsilon}{e^\varepsilon+1}\right)$, where $r=k_++k_-$ is the total number of guesses. This binomial distribution is precisely the distribution of correct guesses we would get if $T$ was obtained by independently performing $(\varepsilon,0)$-DP randomized response on $r$ bits of $S$.
In other words, the theorem says that $(\varepsilon,0)$-DP randomized response is the worst-case algorithm in terms of the number of correct guesses. In particular, this means the theorem is tight (when $\delta=0$)

The binomial distribution is well-concentrated. In particular, for all $\beta \in (0,1)$, we have 
\begin{equation}
    \pr{\check{W} \gets \mathsf{Binomial}\left(r,\frac{e^\varepsilon}{e^\varepsilon+1}\right)}{\check{W} \ge \underbrace{\frac{r \cdot e^\varepsilon}{e^\varepsilon+1} + \sqrt{\frac12 \cdot r \cdot \log(1/\beta)}}_{=v} } \le \beta.\label{eq:bin_hoeffding}
\end{equation}

There is an additional $O(\delta)$ term in the guarantee \eqref{eq:main-inf}. The exact expression for this term is somewhat complex. It is always $\le 2m\delta$, but it is much smaller than this for reasonable parameter values. In particular, for $v$ as in Equation \ref{eq:bin_hoeffding} with $\beta \le 1/r^4$, this term is $\le O(\tfrac{m}{r} \delta)$.

Theorem \ref{thm:main-informal} gives us a hypothesis test: If $\mathcal{A}$ is $(\varepsilon,\delta)$-DP, then the number of correct guesses $W$ is $\le \frac{r \cdot e^\varepsilon}{e^\varepsilon+1} + O(\sqrt{r})$ with high probability. Thus, if the observed number of correct guesses $v$ is larger than this, we can reject the hypothesis that $\mathcal{A}$ satisfies $(\varepsilon,\delta)$-DP. 
We can convert this hypothesis test into a confidence interval (i.e., a lower bound on $\varepsilon$) by finding the largest $\varepsilon$ that we can reject at a desired level of confidence; see Section \ref{sec:hypt_est}.

\section{Related Work}
The goal of privacy auditing is to empirically estimate the privacy provided by an algorithm, typically to accompany a formal privacy guarantee. Early work on auditing has often been motivated by trying to identify bugs in the implementations of differentially private data analysis algorithms~\cite{ding2018detecting, bichsel2018dp}.

Techniques for auditing differentially private machine learning typically rely on conducting some form of membership inference attack \cite{shokri2017membership};\footnote{\citet{shokri2017membership} coined the term ``membership inference attack'' and were the first to apply such attacks to machine learning systems. However, similar attacks were developed for applications to genetic data \cite{homer2008resolving,sankararaman2009genomic,dwork2015robust} and in cryptography \cite{boneh1998collusion,tardos2008optimal}. } these attacks are designed to detect the presence or absence of an individual example in the training set. Essentially, a membership inference attack which achieves some true positive rate (TPR) and false positive rate (FPR) gives a lower bound on the privacy parameter $\varepsilon\ge\log_e(\text{TPR}/\text{FPR})$ (after ensuring statistical validity of the TPR and FPR estimates). 

\citet{jayaraman2019evaluating} use standard membership inference attacks to evaluate different privacy analysis algorithms. \citet{jagielski2020auditing} consider inferring membership of worst-case ``poisoning'' examples to conduct stronger membership inference attacks and understand the tightness of privacy analysis. \citet{nasr2021adversary} measure the tightness of privacy analysis under a variety of threat models, including showing that the DP-SGD analysis is tight in the threat model assumed by the standard DP-SGD analysis.

Improvements to auditing have been made in a variety of directions. For example, \citet{nasr2023tight} and \citet{maddock2022canife} take advantage of the iterative nature of DP-SGD, auditing individual steps to understand privacy of the end-to-end algorithm. Improvements have also been made to the basic statistical techniques for estimating the $\varepsilon$ parameter, for example by using Log-Katz confidence intervals \cite{lu2022general}, Bayesian techniques \cite{zanella2022bayesian}, or auditing algorithms in different privacy definitions \cite{nasr2023tight}.
\citet{andrew2023one} build on the observation that, when performing membership inference, analyzing the case where the data is not included does not require re-running the algorithm; instead we can re-sample the excluded data point; if the data points are i.i.d.~from a nice distribution, this permits closed-form analysis of the excluded case.

A recent heuristic proposed to improve the efficiency of auditing is performing membership inference on multiple examples simultaneously. This heuristic was proposed by \citet{malek2021antipodes}, and evaluated more rigorously by \citet{zanella2022bayesian}. However, this heuristic is not theoretically justified, as the TPR and FPR estimates are not based on independent samples. In our work, we provide a proof of the validity of this heuristic. In fact, with this proof, we show for the first time that standard membership inference attacks, which attack multiple examples per training run, can be used for auditing analysis; prior work using these attacks must make an independence assumption. As a result, auditing can take advantage of progress in the membership inference field \citep{carlini2022membership, wen2022canary}.

\section{Background}
We briefly review some standard background material. Readers may wish to skip to the next section and revisit this only if necessary.
\subsection{Differential Privacy}They 
We recite the definitions of differential privacy and some relevant relaxations. For detailed background, see the tutorial by \citet{vadhan2017complexity} or the textbook by \citet{dwork2014algorithmic}.
\begin{definition}[Differential Privacy {\cite{dwork2006calibrating,dwork2006our}}]
    Let $M : \mathcal{X}^* \to \mathcal{Y}$ be a randomized algorithm, where $\mathcal{X}^* = \bigcup_{n \ge 0} \mathcal{X}^n$.
    We say $M$ is $(\varepsilon,\delta)$-differentially private ($(\varepsilon,\delta)$-DP) if, for all $x,x' \in \mathcal{X}^*$ differing only by the addition or removal of one element, we have \[ \forall S \subset \mathcal{Y} ~~~~~ \pr{}{M(x) \in S} \le e^\varepsilon \cdot \pr{}{M(x') \in S} + \delta.\]
\end{definition}
\begin{definition}[R\'enyi Differential Privacy {\cite{mironov2017renyi}}]
    We say $M : \mathcal{X}^* \to \mathcal{Y}$ is $(\alpha,\check\varepsilon)$-R\'enyi differentially private ($(\alpha,\check\varepsilon)$-RDP) if, for all $x,x' \in \mathcal{X}^*$ differing only by the addition or removal of one element, we have \[ \dr{\alpha}{M(x)}{M(x')} \le \check\varepsilon,\] where $\dr{\alpha}{P}{Q} := \frac{1}{\alpha-1}\log\ex{Y \gets P}{\left(\frac{P(Y)}{Q(Y)}\right)^{\alpha-1}}$ denotes the R\'enyi divergence of order $\alpha$.
\end{definition}
\begin{definition}[Concentrated Differential Privacy {\cite{dwork2016concentrated,bun2016concentrated}}]
    We say $M : \mathcal{X}^* \to \mathcal{Y}$ is $\rho$-zero concentrated differentially private ($\rho$-zCDP) if, for all $x,x' \in \mathcal{X}^*$ differing only by the addition or removal of one element, we have \[ \forall \alpha>1 ~~~~~ \dr{\alpha}{M(x)}{M(x')} \le \alpha \cdot \rho.\]
\end{definition}

\begin{remark}\label{rem:addremove}
    In this paper, we focus on to the addition or removal notion of DP, rather than replacement. (In Appendix \ref{sec:replace_dp}, we consider replacement.)
    Note that, in our theoretical analysis, we consider DP algorithms of the form $M:\{0,1\}^m \to \mathcal{Y}$. In this case, DP is with respect to flipping one of the input bits, as each bit indicates whether some example is included or excluded.
\end{remark}

The main property of DP that we use is invariance under \emph{postprocessing}. That is, if $M : \mathcal{X}^* \to \mathcal{Y}$ satisfies DP and $F : \mathcal{Y} \to \mathcal{Z}$ is an arbitrary function, then $F \circ M : \mathcal{X}^
* \to \mathcal{Z}$ also satisfies DP with the same parameters.

\paragraph{Gaussian Mechanism}
A common method for achieving DP is Gaussian noise addition. The following gives the optimal DP guarantee for the Gaussian mechanism.

\begin{lemma}[{\cite[Theorem 8]{balle2018improving}}]\label{lem:gauss}
    Let $q : \mathcal{X}^* \to \mathbb{R}$ be a function with sensitivity $\Delta : = \sup_{x,x'} |q(x)-q(x')|$. (In the supremum $x,x'\in\mathcal{X}^*$ are restricted to differ only by the addition or removal of one element.)
    Fix $\sigma^2>0$ and let $\rho := \Delta^2/2\sigma^2$.
    Define $M : \mathcal{X}^* \to \mathbb{R}$ by $M(x) = \mathcal{N}(q(x),\sigma^2)$.
    Then, for any $\varepsilon \ge 0$, the algorithm $M$ satisfies $(\varepsilon,\delta)$-DP with \[\delta = \overline\Phi\left(\frac{\varepsilon-\rho}{\sqrt{2\rho}}\right) - e^\varepsilon \cdot \overline\Phi\left(\frac{\varepsilon+\rho}{\sqrt{2\rho}}\right),\] where $\overline\Phi(z) := \pr{Z \gets \mathcal{N}(0,1)}{Z>z} = \frac{1}{\sqrt{2\pi}} \int_z^\infty \exp(-x^2/2) \mathrm{d}x$.
    Furthermore, $M$ satisfies $\rho$-zCDP.
\end{lemma}

\subsection{DP-SGD --  Differentially Private Stochastic Gradient Descent}

The algorithm whose privacy we are most interested in auditing is Differentially Private Stochastic Gradient Descent (DP-SGD, Algorithm \ref{alg:dpsgd}). This is the workhorse of private machine learning both in theory \cite{bassily2014private} and in practice \cite{abadi2016deep}.

\begin{algorithm}[h]
    \caption{DP-SGD --  Differentially Private Stochastic Gradient Descent}\label{alg:dpsgd}
    \begin{algorithmic}[1]
        \State \textbf{Input:} $x \in \mathcal{X}^n$
        \State \textbf{Model:} Loss function $f : \mathbb{R}^d \times \mathcal{X} \to \mathbb{R}$.
        \State \textbf{Parameters:} Number of iterations $\ell \ge 1$, clipping threshold $c>0$, noise multiplier $\sigma>0$, sampling probability $q \in (0,1]$, learning rate $\eta>0$.
        \State Initialize $w_0 \in \mathbb{R}^d$.
        \For{$t = 1, \cdots \ell$}
            \State Sample $S^t \subseteq [n]$ where each $i \in [n]$ is included independently with probability $q$.
            \State Compute $g_i^t = \nabla_{w^{t-1}} f(w^{t-1},x_i) \in \mathbb{R}^d$ for all $i \in S^t$.
            \State Clip $\hat{g}_i^t = \min\left\{1,\frac{c}{\|g_i^t\|_2}\right\} \cdot g_i^t \in \mathbb{R}^d$ for all $i \in S^t$.
            \State Sample $\xi^t  \in \mathbb{R}^d$ from $\mathcal{N}(0,\sigma^2c^2I)$.
            \State Sum $\tilde{g}^t = \xi^t + \sum_{i \in S^t} \hat{g}_i^t  \in \mathbb{R}^d$.
            \State Update $w^t = w^{t-1} - \eta \cdot \tilde{g}^t \in \mathbb{R}^d$.
        \EndFor
        \State \textbf{Output:} $w^0,w^1,\cdots,w^\ell$.
    \end{algorithmic}
\end{algorithm}

DP-SGD satisfies differential privacy. Much ink has been spilled precisely quantifying its privacy properties \cite[etc.]{mironov2019r,wang2019subsampled,koskela2020computing,gopi2021numerical,zhu2022optimal}. A simple guarantee is the following.

\begin{proposition}[\cite{mironov2019r,steinke2022composition}]\label{prop:rdp-dpsgd}
    DP-SGD (Algorithm \ref{alg:dpsgd}) satisfies $(2,\check\varepsilon)$-RDP for \[\check\varepsilon = \ell \cdot \log\left(1 + q^2 \cdot \left(\exp(1/\sigma^2)-1\right)\right) \approx \ell \cdot q^2 \cdot \frac{1}{\sigma^2}.\]
\end{proposition}

If $\check\varepsilon \le 1$, then DP-SGD should provide meaningful privacy protection. 
In particular, $(2,\check\varepsilon)$-RDP implies that membership inference has a maximum accuracy (in the balanced case) of %
\begin{equation}
    \frac12 +\frac12 \sqrt{\frac{e^{\check\varepsilon}-1}{e^{\check\varepsilon}+3}} \approx \frac12 + \frac14 \sqrt{\check\varepsilon}.
\end{equation}
Our goal is to audit this guarantee.

\subsection{Hypothesis Testing \& Statistical Estimation}\label{sec:hypt_est}

Our goal is to estimate the privacy parameters of the algorithm that we are auditing. As prior work has noted~\cite{ding2018detecting, jagielski2020auditing}, this task can be framed as statistical estimation, with a goal of outputting a statistical lower bound on the privacy parameters. These lower bounds will have a corresponding confidence level, roughly representing the probability that the lower bound could have been produced even when analyzing an algorithm with perfect privacy. As empirical methods, it is impossible to have 100\% confidence in our methods, so we will generally use 95\% confidence in our experiments, comparable to the use of $p<0.05$ in science literature.

To be precise, our auditor runs the algorithm $M$ and outputs $\varepsilon_
{\text{LB}} \ge 0$ with the following guarantee. If $M$ satisfies $(\varepsilon_{\text{true}},\delta)$-DP, then, with probability at least $1-\beta$, we have $\varepsilon_{\text{LB}} \le \varepsilon_{\text{true}}$. Here $1-\beta$ is the confidence level and $\delta \ge 0$ is fixed.
Note that this is a frequentist guarantee, rather than a Bayesian guarantee. That is, the probability is with respect to our auditing procedure, rather than a statement about our beliefs about $M$. 

We can also view this in terms of hypothesis testing. Here we start with a ``null hypothesis'' that $M$ satisfies $(\varepsilon_{\text{null}},\delta)$-DP and the auditor's goal is to test this hypothesis by running $M$.
If the auditor rejects this null hypothesis, then this gives us a lower bound $\varepsilon_{\text{LB}} = \varepsilon_{\text{null}}$.

The difference between hypothesis testing and statistical estimation is that a hypothesis test starts with a given $\varepsilon_{\text{null}}$ and outputs a binary decision to reject or not, while an estimator outputs a number $\varepsilon_{\text{LB}}$. However, we can convert between these:

\begin{lemma}\label{lem:testci}
    For each $M$, let $A_M \in \Omega$ be a random variable and let $P_M \in \mathbb{R}$ be a fixed number.
    For each $\varepsilon,\beta > 0$, let $T_{\varepsilon,\beta} \subset \Omega$ satisfy 
    \begin{equation}
        \forall M ~~~~~ \left( P_M = \varepsilon ~~\implies~~ \pr{}{A_M \in T_{\varepsilon,\beta}}\le\beta \right) . \label{eq:hyptest}
    \end{equation}
    Further suppose that, if $\varepsilon_1 \le \varepsilon_2$, then $T_{\varepsilon_1,\beta} \supset T_{\varepsilon_2,\beta}$.
    Then, for all $M$ and all $\beta > 0$,
    \begin{equation}
        \pr{}{P_M \ge \sup\left\{\varepsilon > 0 : A_M \in T_{\varepsilon,\beta} \right\}} \ge 1 - \beta. \label{eq:ci}
    \end{equation}
\end{lemma}
\begin{proof}
    Fix a realization of $A_M$ and suppose $P_M < \sup\left\{\varepsilon > 0 : A_M \in T_{\varepsilon,\beta} \right\}$. Then there exists some $\varepsilon \ge P_M$ with $A_M \in T_{\varepsilon,\beta}$ and, hence, \[ A_M \in \bigcup_{\varepsilon \ge P_M} T_{\varepsilon,\beta} = T_{P_M,\beta}.\]
    The equality above follows from our monotonicity assumption on $T$.
    Thus \[\pr{}{P_M < \sup\left\{\varepsilon > 0 : A_M \in T_{\varepsilon,\beta} \right\}} \le \pr{}{A_M \in T_{P_M,\beta}} \le \beta,\] as required.
\end{proof}
To interpret Lemma \ref{lem:testci}, $M$ is an algorithm and $P_M$ is the ``true'' privacy parameter $\varepsilon$ that it satisfies. (We're considering $\delta$ to be fixed.) The random variable $A_M$ is the output of our auditing procedure applied to $M$. (This is our test statistic in the language of hypothesis testing.) The hypothesis test's rejection set is $T_{\varepsilon,\beta}$ and Equation \ref{eq:hyptest} guarantees that, if $M$ is indeed $(\varepsilon,\delta)$-DP (i.e., the null hypothesis is true), then the probability that we reject the null hypothesis is at most $\beta$.
Equation \ref{eq:ci} then shows how to estimate the true privacy parameter $P_M$ from $A_M$; we simply take the largest $\varepsilon$ for which we can reject the corresponding null hypothesis.

Note that Lemma \ref{lem:testci} needs to make a technical monotonicity assumption. In our setting this simply means that, if a given realization of the test statistic $A_M$ allows us to reject the null hypothesis that $M$ is $(\varepsilon_2,\delta)$-DP and $\varepsilon_1 \le \varepsilon_2$, then we can also reject the null hypothesis that $M$ is $(\varepsilon_1,\delta)$-DP. 

\subsection{Stochastic Dominance}

In our theoretical analysis we use the concept of stochastic dominance. Specifically, we use this to formalize the ``worst-case'' DP algorithm for auditing.

\begin{definition}[Stochastic Dominance]
    Let $X, Y \in \mathbb{R}$ be random variables. We say $X$ is stochastically dominated by $Y$ (or $Y$ stochastically dominates $X$) if $\pr{}{X>t} \le \pr{}{Y>t}$ for all $t \in \mathbb{R}$. 
    Equivalently, $X$ is stochastically dominated by $Y$ if there exists a coupling (i.e., a joint distribution that matches the marginal distributions of $X$ and $Y$) such that $\pr{}{X \le Y}=1$.
\end{definition}

Stochastic dominance is preserved under sums/convolutions:

\begin{lemma}\label{lem:sumdominance}
    Suppose $X_1$ is stochastically dominated by $Y_1$.
    Suppose that, for all $x \in \mathbb{R}$, the conditional distribution $X_2|X_1=x$ is stochastically dominated by $Y_2$.
    Assume that $Y_1$ and $Y_2$ are independent.
    Then $X_1+X_2$ is stochastically dominated by $Y_1+Y_2$.
\end{lemma}
\begin{proof}
    For all $t \in \mathbb{R}$, we have
    \begin{align*}
        \pr{}{X_1+X_2>t} &= \ex{X_1}{\prc{X_2}{X_2>t-X_1}{X_1}} \\
        &\le \ex{X_1}{\pr{Y_2}{Y_2 > t - X_1}} \tag{$Y_2$ dominates $X_2|X_1$} \\
        &= \ex{Y_2}{\pr{X_1}{X_1 > t - Y_2}} \\
        &\le \ex{Y_2}{\pr{Y_1}{Y_1 > t - Y_2}} \tag{$Y_1$ dominates $X_1$ \& independence} \\
        &= \pr{}{Y_1 + Y_2 > t}.
    \end{align*}
\end{proof}

\section{Theoretical Analysis}\label{sec:theory}

\begin{figure}[h]
    \centering
    \includegraphics[width=0.75\textwidth]{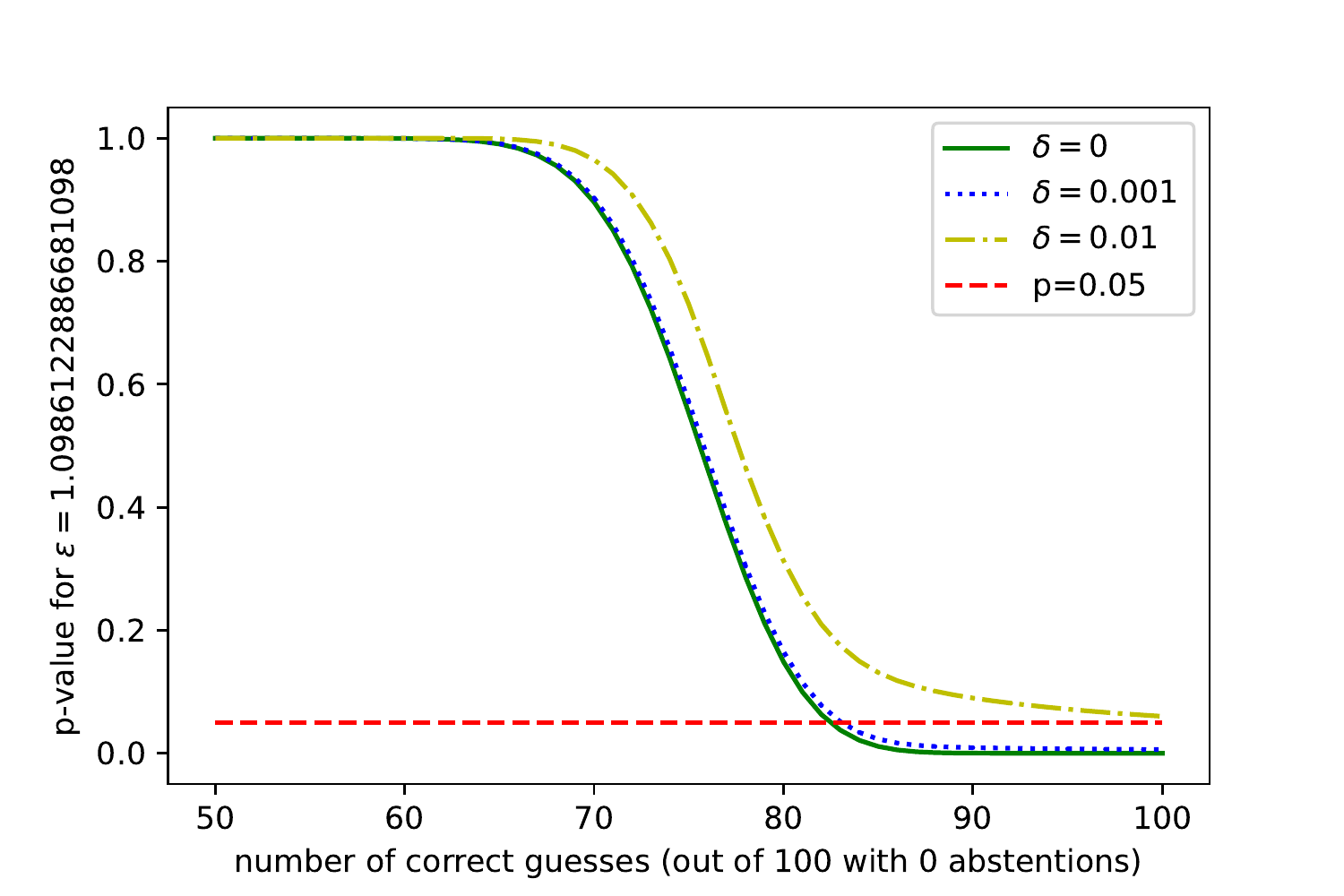}
    \caption{Theorem \ref{thm:main}'s p-value as the number of correct guesses changes for fixed $\varepsilon=\log 3$ (i.e., ideally 75\% of guesses correct). The total number of examples and guesses is 100.}
    \label{fig:v-p}
\end{figure}
\begin{figure}[h]
    \centering
    \includegraphics[width=0.75\textwidth]{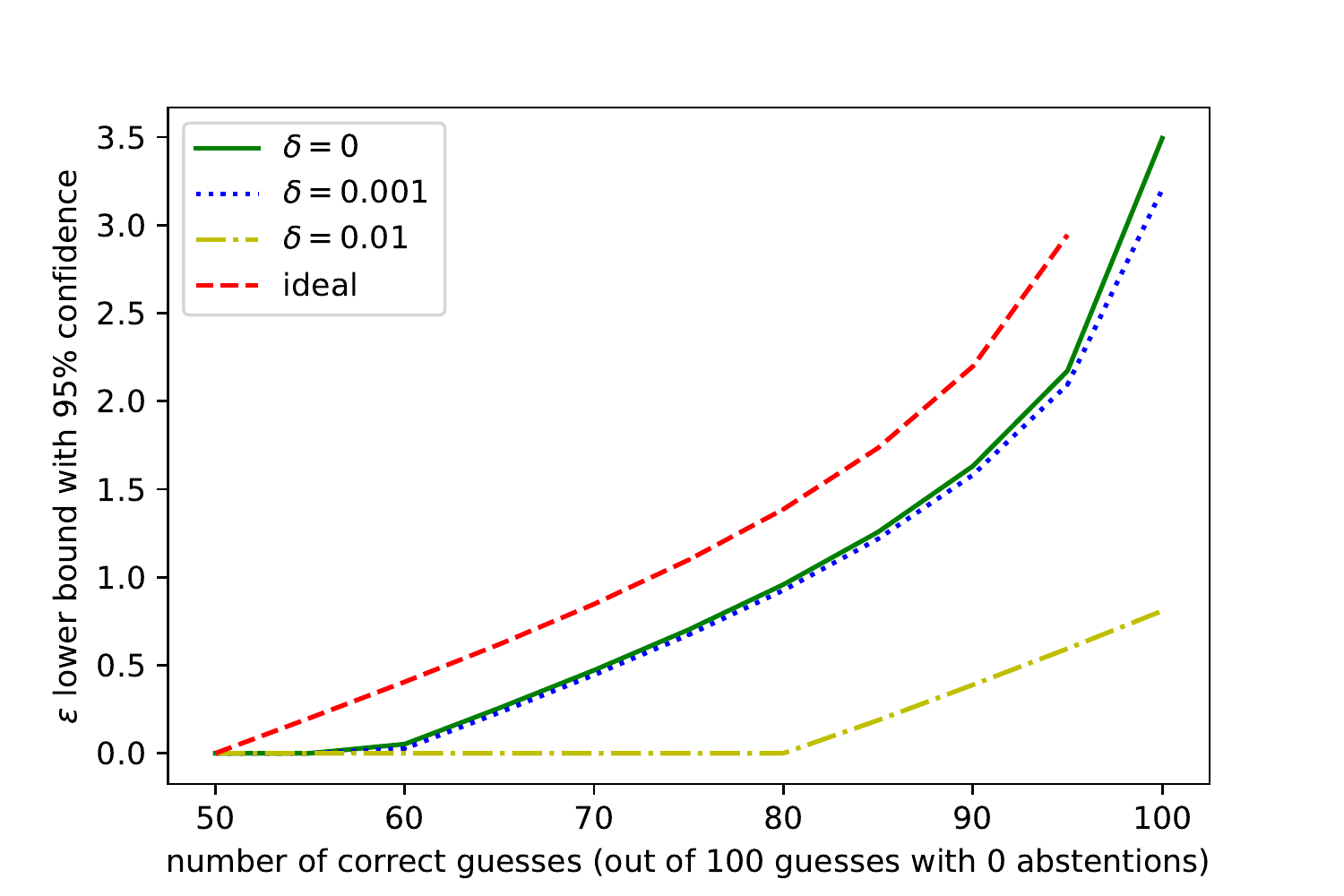}
    \caption{Lower bound on the privacy parameter $\varepsilon$ given by Theorem \ref{thm:main} with 95\% confidence as the number of correct guesses changes. The total number of examples and guesses is 100. For comparison, we plot the \texttt{ideal} $\varepsilon$ that gives $100 \cdot \frac{ e^\varepsilon}{e^\varepsilon+1}$ correct guesses. %
    }
    \label{fig:v-eps}
\end{figure}
\begin{figure}[h]
    \centering
    \includegraphics[width=0.75\textwidth]{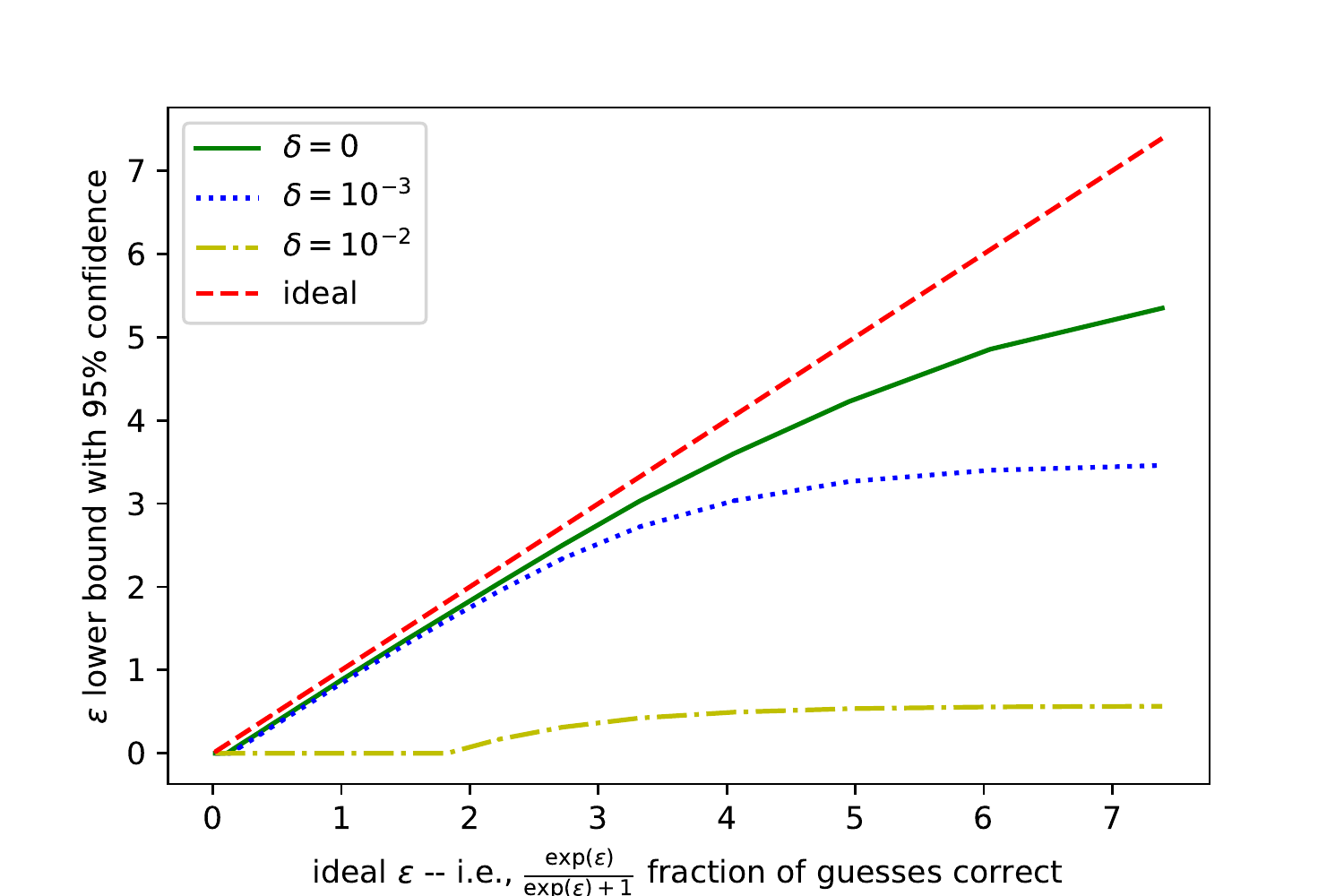}
    \caption{Lower bound on the privacy parameter $\varepsilon$ given by Theorem \ref{thm:main} with 95\% confidence as the number of correct guesses changes. The total number of examples and guesses is 1000 (with no abstentions). Here we plot the ideal $\varepsilon$ on the horizontal axis, so that the number of correct guesses is $1000 \cdot \frac{e^\varepsilon}{e^\varepsilon+1}$.}
    \label{fig:eps-eps}
\end{figure}

To analyze the results of our audit, we leverage the connection between DP and generalization \citep{dwork2015preserving,dwork2015generalization,bassily2016algorithmic,rogers2016max,jung2019new,steinke2020reasoning}.
Unfortunately, directly applying the existing results from the literature is unlikely to yield meaningful results, as the constants are not optimal.
Thus we provide an analysis of DP's generalization guarantees that is suitable for our application and which has sharp constants.

We consider the following formalism. The algorithm $M : \{-1,+1\}^m \to \mathbb{R}^m$ takes in a vector of bits and outputs a vector of ``guesses''. 
Each input bit indicates whether or not a particular example is included in or excluded from the dataset. 
In particular, the DP guarantee ensures that the outputs are indistinguishable if we flip one bit, which corresponds to adding or removing the corresponding data point.
Each coordinate of the output is a guess for the corresponding input bit; the sign of the score should match the corresponding input bit, while the magnitude is a reflection of the confidence.

The algorithm $M$ represents both the ``real'' algorithm (e.g., DP-SGD) and the auditor which postprocesses the output of the real algorithm into guesses. In this formalism, the examples themselves are considered fixed and not part of the input -- i.e., the examples are ``hardcoded'' into $M$. The algorithm $M$ is an abstraction for our analysis, rather than a realistic system. 

We evaluate the quantity \[W:=\sum_i^m \max\{ 0 , T_i \cdot S_i \},\] where $S$ is uniform on $\{-1,+1\}^m$ and $T=M(S)$. If $T_i$ and $S_i$ disagree in sign (i.e., the guess is wrong), then $\max\{ 0 , T_i \cdot S_i \} = 0$; if they agree (i.e., the guess is right), then $\max\{ 0 , T_i \cdot S_i \}=|T_i|$. That is, $W$ increases when we guess correctly and the increase is proportional to how much ``weight'' we placed on that guess. 
The auditor seeks to maximize $W$ and then we compare it to a baseline that is consistent with DP. (The analysis in this section focuses on computing this baseline.) Incorrect guesses do not increase $W$, but they do increase the baseline. Note that we can guess $T_i=0$, which amounts to abstaining from making a guess; this doesn't increase $W$, but also doesn't increase the baseline.

Our formalism is inspired by that of \citet{steinke2020reasoning}, who also restrict to binary inputs.
In contrast, most of the work connecting DP and generalization does not do this. The benefit of restricting to binary inputs which represent inclusion or exclusion of a data point is that it simplifies our analysis.

\subsection{Pure DP Analysis}\label{subsec:pure}
We first consider the pure DP ($\delta=0$) case, as it is considerably simpler than the general case.
We follow the analysis of \citet{jung2019new} with some refinement. Specifically, rather than relying on a Hoeffding bound, we show that it is stochastically dominated by a Binomial distribution.
This result is tight -- i.e., if $M$ independently performs a randomized response for each input bit, then the inequality becomes an equality. 

\begin{proposition}[Pure DP Version of Main Result]\label{prop:pure}
    Let $M : \{-1,+1\}^m \to \mathbb{R}^m$ satisfy $(\varepsilon,0)$-DP.
    Let $S \in \{-1,+1\}^m$ be uniformly random. Let $T=M(S) \in \mathbb{R}^m$.
    Then, for all $v \in \mathbb{R}$ and all $t \in \mathbb{R}^m$ in the support of $T$,\footnote{To be precise, this holds with probability 1, but may fail for $t$ in a set of measure zero under $T$. Note that $S \gets \{-1,1\}^m$ denotes that $S$ is uniform on the set $\{-1,1\}^m$ and $\check S \gets \mathsf{Bernoulli}\left(\frac{e^\varepsilon}{e^\varepsilon+1}\right)$ denotes that $\check S \in \{0,1\}^m$ has a product distribution with each coordinate having expectation $\frac{e^\varepsilon}{e^\varepsilon+1}$.}
    \[\prc{S \gets \{-1,1\}^m \atop T \gets M(S)}{\sum_i^m \max\{ 0 , T_i \cdot S_i\} \ge  v }{T=t} \le \pr{\check{S} \gets \mathsf{Bernoulli}\left(\frac{e^\varepsilon}{e^\varepsilon+1}\right)^m}{\sum_i^m \check{S}_i \cdot |t_i|  \ge  v} =: \beta(m,\varepsilon,v,t). \]%
\end{proposition}

Proposition \ref{prop:pure} is Bayesian: We condition on the output and then consider the probability that each guess was right. The vector $\check{S}$ should be seen as indicating whether each guess was right. The proposition says that, in the worst case, each guess is correct independently with probability $\frac{e^\varepsilon}{e^\varepsilon+1}$.

How do we use this result? Suppose we have conducted an audit and observed $s$ and $t$ as the output of Algorithm \ref{alg:audit}. Let $v= \sum_i^m \max\{ 0 , s_i \cdot t_i \}$. Following Lemma \ref{lem:testci}, we choose a desired confidence $1-\beta<1$ (e.g., $\beta=0.05$) and then we choose $\varepsilon\ge 0$ so that $\beta(m,\varepsilon,v,t)=\beta$. Then this value of $\varepsilon$ is our lower bound.

In the language of hypothesis testing, $W=\sum_i^m \max\{ 0 , T_i \cdot S_i\}$ is the test statistic and our null hypothesis is that $M$ is $\varepsilon$-DP. Under the null hypothesis we have $\pr{}{W \ge v}\le \beta(m,\varepsilon,v,t)$. Thus, if $v$ is the observed value of the test statistic, then $\beta(m,\varepsilon,v,t)$ is our p-value. And we can reject the null hypothesis if, say, $\beta(m,\varepsilon,v) \le 0.05$.

\begin{proof}
    Fix some $t\in\mathbb{R}^m$. 
    We now analyze the distribution of $S$ conditioned on $M(S)=t$. Note that the unconditional distribution of $S$ is uniform on $\{-1,+1\}^m$ and $M$ is $(\varepsilon,0)$-DP.
    We perform the analysis one bit at a time. Fix some $i \in [m]$ and $s_{<i} \in \{-1,+1\}^{i-1}$.
    By Bayes' law and $(\varepsilon,0)$-DP, 
    \begin{align*}
        &\pr{}{S_i=1 | M(S)=t, S_{<i}=s_{<i}} \\
        &~= \frac{\pr{}{M(S)=t|S_i=1, S_{<i}=s_{<i}}\cdot \pr{}{S_i=1|S_{<i}=s_{<i}}}{\pr{}{M(S)=t|S_{<i}=s_{<i}}} \\
        &~= \frac{\pr{}{M(S)=t|S_i=1,S_{<i}=s_{<i}}\cdot \pr{}{S_i=1}}{\pr{}{M(S)=t|S_i=1,S_{<i}=s_{<i}}\cdot \pr{}{S_i=1} + \pr{}{M(S)=t|S_i=-1,S_{<i}=s_{<i}}\cdot \pr{}{S_i=-1}}\\
        &~= \frac{\pr{}{M(S)=t|S_i=1,S_{<i}=s_{<i}}\cdot \frac12 }{\pr{}{M(S)=t|S_i=1,S_{<i}=s_{<i}}\cdot \frac12 + \pr{}{M(S)=t|S_i=-1,S_{<i}=s_{<i}}\cdot \frac12}\\
        &~= \frac{1}{1 + \pr{}{M(S)=t|S_i=-1,S_{<i}=s_{<i}}/\pr{}{M(S)=t|S_i=1,S_{<i}=s_{<i}}}\\
        &~\in \left[ \frac{1}{1+e^{\varepsilon}}, \frac{1}{1+e^{-\varepsilon}} \right].
    \end{align*}
    Thus  $ \pr{}{S_i = \sign(T_i) | T=t, S_{<i}=s_{<i}} \le \frac{1}{1+e^{-\varepsilon}} = \frac{e^\varepsilon}{e^\varepsilon+1}$.
    
    With this in hand, we can prove the result by induction. We assume inductively that $W_{m-1} := \sum_i^{m-1}  \max\{ 0 , T_i \cdot S_i \}$ is stochatiscally dominated by $\check{W}_{m-1} := \sum_i^{m-1} \check{S}_{i} \cdot |t_{i}| $ where $\check{S} \gets \mathsf{Bernoulli}\left(\frac{e^\varepsilon}{e^\varepsilon+1}\right)^{m-1} $.
    As above, conditioned on the value of $W_{m-1}$, the variable $\max\{ 0 , T_m \cdot S_m \} = |T_m| \cdot \mathbb{I}[S_m = \sign(T_m)]$ is stochastically dominated by $|T_m| \cdot \mathsf{Bernoulli}\left(\frac{e^\varepsilon}{e^\varepsilon+1}\right)$. By Lemma \ref{lem:sumdominance}, $W_m = W_{m-1}+\max\{ 0 , T_m \cdot S_m \}$ is stochastically dominated by $\check{W}_{m} := \sum_i^m \check{S}_{i} \cdot |t_{i}| $ where $\check{S} \gets \mathsf{Bernoulli}\left(\frac{e^\varepsilon}{e^\varepsilon+1}\right)^{m} $.
\end{proof}

\subsection{Approximate DP Analysis}\label{sec:approx}

We extend the pure DP analysis (\S\ref{subsec:pure}) to approximate DP ($\delta>0$).
This becomes quite messy. In the pure DP case, we can condition on an arbitrary output $t$. In the approximate DP case, some outputs are ``bad'' in the sense that the privacy loss is unbounded. 
To handle this we do two things: First, we require the guesses to be bounded (i.e., $T \in [-1,+1]^m$ instead of $T \in \mathbb{R}^m$), which ensures that bad outputs cannot skew things too much. Second, the guarantees we prove have an additional failure probability that depends on $\delta$.

Our analysis most closely resembles that of \citet{rogers2016max}. Essentially, we repeat the analysis for the pure DP case, but add some failure events, and carefully account for how much they can distort the results.

\begin{theorem}[Main Result]\label{thm:main}
    Let $M : \{-1,+1\}^m \to [-1,+1]^m$ satisfy $(\varepsilon,\delta)$-DP.
    Let $S \in \{-1,+1\}^m$ be uniformly random. Let $T=M(S) \in [-1,+1]^m$.
    Then, for all $v \in \mathbb{R}$,
    \begin{equation}
    \pr{S \gets \{-1,+1\}^m \atop T \gets M(S)}{\sum_i^m \max\{0, T_i \cdot S_i\} \ge  v} \le \beta + \alpha\cdot 2m\cdot \delta, 
    \end{equation}
    where     
    \begin{align}
        \beta &= \pr{\check{W}^*}{\check{W}^* \ge  v},\\
        \alpha &= \max \left\{ \frac{1}{i} \left( \pr{\check{W}^*}{\check{W}^* \ge  v-i} - \beta \right) : i \in \{1,2,\cdots,m\} \right\}.
    \end{align}
    Here $\check{W}^*$ is any distribution on $\mathbb{R}$ that stochastically dominates $\check{W}(t) := \sum_i^m \check{S}_i |t_i|$ for $\check{S} \gets \mathsf{Bernoulli}\left(\frac{e^\varepsilon}{e^\varepsilon+1}\right)^m$ for all $t$ in the support of $T$.

\end{theorem}

To evaluate the bound of Theorem \ref{thm:main}, we need to identify $\check{W}^*$ and compute its distribution. We can set $\pr{\check{W}^*}{\check{W}^* \ge v} = \sup_{t \in \mathsf{support}(T)} \pr{\check{W}}{\check{W}(t) \ge  v}$. This can be difficult to compute, depending on what we know about the support of $T$.
If the support of $T$ is nice, we can compute this explicitly; e.g., see Corollary \ref{cor:ternary}.
There are other things we can do. For example, if we have bounds on $\sup_{t \in \mathsf{support}(T)} \|t\|_2$ and $\sup_{t \in \mathsf{support}(T)} \|t\|_1$, then we can use a concentration inequality to bound $\sup_{t \in \mathsf{support}(T)} \pr{\check{W}}{\check{W}(t) \ge  v}$ and then use this bound as the distribution of $\check{W}^*$. This yields the following corollary.

\begin{corollary}[Analytic Version of Main Result]\label{cor:hoeffding}
    Let $M : \{-1,+1\}^m \to [-1,+1]^m$ satisfy $(\varepsilon,\delta)$-DP.
    Let $S \in \{-1,+1\}^m$ be uniformly random. Let $T=M(S) \in [-1,+1]^m$.
    Suppose $\pr{}{\|T\|_2 \le r_2}=1$ and $\pr{}{\|T\|_1 \le r_1}=1$.
    Then, for all $v \ge \frac{e^\varepsilon}{e^\varepsilon+1} \cdot r_1 + 2$, we have
    \begin{align}
        \pr{S \gets \{-1,+1\}^m \atop T \gets M(S)}{\sum_i^m \max\{0, T_i \cdot S_i\} \ge  v} &\le f(v) +  2m\cdot \delta \cdot  \max \left\{ \frac{f(v-i) - f(v)}{i} : i \in [m] \right\} \label{eq:cor:hoeffding:1}\\
        &\le f(v) + 2m\delta \cdot \max\left\{ \frac{2}{v - \frac{e^\varepsilon}{e^\varepsilon+1} r_1} , f\left(\frac12\left(v+\frac{e^\varepsilon}{e^\varepsilon+1} r_1\right)\right) \right\},\label{eq:cor:hoeffding:2}
    \end{align}
    where
    \[f(v) := \left\{ \begin{array}{cl} \exp\left(\frac{-2}{r_2^2}\left(v - \frac{e^\varepsilon}{e^\varepsilon+1} r_1\right)^2\right) & \text{ if } v \ge \frac{e^\varepsilon}{e^\varepsilon+1} r_1 \\ 1 &  \text{ if } v < \frac{e^\varepsilon}{e^\varepsilon+1} r_1\end{array} \right\}.\]
\end{corollary}
In particular, if we substitute $v = \frac{e^\varepsilon}{e^\varepsilon+1} r_1 + r_2 \cdot \sqrt{\frac12 \log(1/\beta)}$ into Equation \ref{eq:cor:hoeffding:2}, we get
\begin{equation}
    \pr{S \gets \{-1,+1\}^m \atop T \gets M(S)}{\sum_i^m \max\{0, T_i \cdot S_i\} \ge  v } \le \beta + 2m\cdot \delta\cdot \max\left\{\frac{1}{r_2\sqrt{\frac12\log(1/\beta)}} , \beta^{1/4}\right\}.
\end{equation}
\begin{proof}
    Fix an arbitrary $t$ the support of $T$.
    Define $\check{W}(t) := \sum_i^m \check{S}_i |t_i|$ for $\check{S} \gets \mathsf{Bernoulli}\left(\frac{e^\varepsilon}{e^\varepsilon+1}\right)^m$.
    Then $\ex{}{\check{W}(t)} = \frac{e^\varepsilon}{e^\varepsilon+1} \|t\|_1$.
    By Hoeffding's inequality, for all $\lambda \ge 0$, \[\pr{}{\check{W}(t) \ge \frac{e^\varepsilon}{e^\varepsilon+1} \|t\|_1 + \lambda} \le \exp\left(\frac{-2\lambda^2}{\|t\|_2^2}\right).\]
    Now define $\check{W}^*$ by  \[\pr{}{\check{W}^* \ge v} :=  f(v) := \left\{ \begin{array}{cl} \exp\left(\frac{-2}{r_2^2}\left(v - \frac{e^\varepsilon}{e^\varepsilon+1} r_1\right)^2\right) & \text{ if } v \ge \frac{e^\varepsilon}{e^\varepsilon+1} r_1 \\ 1 &  \text{ if } v < \frac{e^\varepsilon}{e^\varepsilon+1} r_1\end{array} \right\}.\]
    Since $\check{W}^*$ stochastically dominates $\check{W}(t)$ for all $t$ in the support of $T$, we can apply Theorem \ref{thm:main} to obtain the first part of the result \eqref{eq:cor:hoeffding:1}.
    
    Next, for any $c \ge 1$, we have
    \begin{align*}
        &\max \left\{ \frac{f(v-i) - f(v)}{i} : i \in [m] \right\} \\
        &\le \max \left\{ \frac{f(v-x)}{x} : x \in [1,\infty) \right\} \tag{$f(v)\ge 0$ and $[m] \subset [1,\infty)$}\\
        &= \max\left\{ \max \left\{ \frac{f(v-x)}{x} : x \in [1,c] \right\} , \max \left\{ \frac{f(v-x)}{x} : x \in [c,\infty) \right\} \right\} \\
        &\le \max\left\{ \max \left\{ \frac{f(v-x)}{1} : x \in [1,c] \right\} , \max \left\{ \frac{1}{x} : x \in [c,\infty) \right\} \right\} \\
        &= \max\left\{ f(v-c) , \frac1c \right\}.
    \end{align*}
    Setting $c = \frac12\left(v-\frac{e^\varepsilon}{e^\varepsilon+1} r_1\right)$ yields the second part of the result \eqref{eq:cor:hoeffding:2}
\end{proof}

In the next corollary we restrict $M$ to ternary outputs, so it must either guess ($T_i = \pm 1$) or abstain ($T_i=0$). We bound the number of guesses by $r$.
In this case the dominating distribution $\check{W}^*$ is a binomial distribution, which is relatively easy to compute.
This is the form of Theorem \ref{thm:main} that we use in all of our experimental results.
We provide pseudocode in Appendix \ref{app:code}.

\begin{corollary}[Ternary Guesses]\label{cor:ternary}
    Let $M : \{-1,+1\}^m \to \{-1,0,+1\}^m$ satisfy $(\varepsilon,\delta)$-DP.
    Let $S \in \{-1,+1\}^m$ be uniformly random. Let $T=M(S) \in \{-1,0,+1\}^m$.
    Suppose $\pr{}{\|T\|_1 \le r}=1$.
    Then, for all $v \in \mathbb{R}$,
    \[\pr{S \gets \{-1,+1\}^m \atop T \gets M(S)}{\sum_i^m \max\{0, T_i \cdot S_i\} \ge  v} \le f(v) +  2m\cdot \delta \cdot  \max \left\{ \frac{f(v-i) - f(v)}{i} : i \in \{1,2,\cdots,m\} \right\},\]
    where
    \[f(v) := \pr{\check{W} \gets \mathsf{Binomial}\left(r,\frac{e^\varepsilon}{e^\varepsilon+1}\right)}{\check{W} \ge v}.\]
\end{corollary}

Now we delve into the proof of Theorem \ref{thm:main}.
We use a decomposition result of \citet{kairouz2015composition} (see also \cite{murtagh2015complexity} \& \cite[Corollary 24]{steinke2022composition}).
\begin{lemma}\label{lem:kov}
    Let $P$ and $Q$ be probability distributions over $\mathcal{Y}$. Fix $\varepsilon, \delta \ge 0$. Suppose that, for all measurable $S \subset \mathcal{Y}$, we have $P(S) \le e^\varepsilon \cdot Q(S) + \delta$ and $Q(S) \le e^\varepsilon P(S) + \delta$.
    
    Then there exist $\delta' \in [0,\delta]$ and distributions $P'$, $Q'$, $P''$, and $Q''$ over $\mathcal{Y}$ such that the following three properties are all satisfied.
    First, we can express $P$ and $Q$ as convex combinations:
    \begin{align*}
        P &= (1-\delta') P' + \delta' P''  ,\\
        Q &= (1-\delta') Q' + \delta' Q'' .
    \end{align*}
    Second, for all measurable $S \subset \mathcal{Y}$, we have $e^{-\varepsilon} P'(S)\le Q'(S) \le e^\varepsilon P'(S)$.
    Third, there exist measurable $S,T \subset \mathcal{Y}$ such that $P''(S)=1$, $Q''(T)=1$, $\forall S' \subset S ~ P(S') \ge Q(S')$, and $\forall T' \subset T ~ Q(T') \ge P(T')$.
\end{lemma}
    \begin{proof}
        This proof follows that of \citet{steinke2022composition}. We begin with some formalities:
        Fix some base measure such that $P$ and $Q$ are absolutely continuous with respect to the base measure. (If $P$ and $Q$ are discrete distributions, this can be the counting measure. If they are continuous distributions, this can be the Lebesgue measure. In general, $P+Q$ serves as such a measure.) For $y \in \mathcal{Y}$, let $P(y)$ and $Q(y)$ denote the Radon-Nikodym derivative of $P$ and, respectively, $Q$ with respect to this base measure. %
    
        If $e^{-\varepsilon} \cdot Q(S) \le P(S) \le e^\varepsilon \cdot Q(S)$ for all measurable $S$, then the result follows trivially by setting $\delta'=0$, $P'=P$ and $Q'=Q$, and choosing $P''$ and $Q''$ to be arbitrary distributions supported on $S=\{y \in \mathcal{Y} : P(y) \ge Q(y)\}$ and $T=\{ y \in \mathcal{Y} : P(y) \le Q(y)\}$ respectively. Thus we assume that this is not the case and, hence, that $\delta>0$ and $\tvd{P}{Q}>0$.
        
        Similarly, if $\delta \ge 1$ and $\tvd{P}{Q}=1$, then the result follows trivially by setting $\delta'=1$, $P''=P$, $Q''= Q$, and $P'=Q'$ arbitrary. Thus we assume that $\min\{\delta, \tvd{P}{Q}\} < 1$.
        
        Fix $\varepsilon_1, \varepsilon_2 \in [0,\varepsilon]$ to be determined later.
        Define distributions $P'$, $P''$, $Q'$, and $Q''$ (in terms of their Radon-Nikodym derivatives) as follows.
        For all points $y \in \mathcal{Y}$,
        \begin{align*}
            P'(y) &= \frac{\min\{ P(y) , e^{\varepsilon_1} \cdot Q(y) \}}{1-\delta_1}, \\
            P''(y) &= \frac{P(y) - (1-\delta_1)P'(y)}{\delta_1} = \frac{\max\{0, P(y) - e^{\varepsilon_1} \cdot Q(y)\}}{\delta_1}, \\
            Q'(y) &= \frac{\min\{ Q(y), e^{\varepsilon_2} \cdot P(y) \}}{1-\delta_2}, \\
            Q''(y) &= \frac{Q(y) - (1-\delta_2) Q'(y)}{\delta_2} = \frac{\max\{0, Q(y) - e^{\varepsilon_2} \cdot P(y)\}}{\delta_2},
        \end{align*}
        where $\delta_1,\delta_2\in(0,1)$ are appropriate normalizing constants. (We will choose $\varepsilon_1$ to avoid $\delta_1\in\{0,1\}$ and, likewise, we will choose $\varepsilon_2$ to avoid $\delta_2\in\{0,1\}$.)
        
        By construction, $(1-\delta_1)P' + \delta_1 P'' = P$ and $(1-\delta_2)Q'+\delta_2 Q'' = Q$, so the first property is satisfied.
        Note that $P''$ is supported on $S = \{y \in \mathcal{Y} : P(y) > e^{\varepsilon_1} \cdot Q(y)\}$ and $Q''$ is supported on $T=\{ y \in \mathcal{Y} : Q(y) > e^{\varepsilon_2} \cdot P(y)\}$, which implies the third property.
        
        If $0 < \delta_1=\delta_2 \le \delta$, then we have the appropriate decomposition (with $\delta'=\delta_1=\delta_2$) and, for all $y \in \mathcal{Y}$, we have \[e^{-\varepsilon} \le e^{-\varepsilon_2} \le \frac{P'(y)}{Q'(y)} = \frac{\min\{P(y), e^{\varepsilon_1} \cdot Q(y)\}}{\min\{Q(y), e^{\varepsilon_2} \cdot P(y)\}} \le e^{\varepsilon_1} \le e^\varepsilon ,\] as required for the second property.

        It only remains to show that we can ensure that $0 < \delta_1 = \delta_2 \le \delta$ by appropriately setting $\varepsilon_1, \varepsilon_2 \in [0,\varepsilon]$.
        We have \[\delta_1 = \int_{\mathcal{Y}} \max\{0, P(y) - e^{\varepsilon_1} \cdot Q(y)\} \mathrm{d}y = \int_S P(y) - e^{\varepsilon_1} \cdot Q(y) \mathrm{d}y = P(S) - e^{\varepsilon_1} Q(S),\] where $S = \{y \in \mathcal{Y} : P(y) \ge e^{\varepsilon_1} \cdot Q(y)\}$. If $\varepsilon_1=\varepsilon$, then $\delta_1 \le \delta$ by assumption. If $\varepsilon_1 = 0$, then $\delta_1=\tvd{P}{Q}>0$.
         By decreasing $\varepsilon_1$, we continuously increase $\delta_1$. Thus, by starting at $\varepsilon_1=\varepsilon$ and decreasing $\varepsilon_1$ until either $\varepsilon_1=0$ or $\delta_1=\delta$, we can pick $\varepsilon_1 \in [0,\varepsilon]$ such that $\delta_1 = \min\{ \delta, \tvd{P}{Q} \}\in(0,1)$. Similarly, we can pick $\varepsilon_2 \in [0,\varepsilon]$, such that $\delta_2 = \min\{ \delta, \tvd{P}{Q} \}$.
    \end{proof}

We need a Bayesian version of this decomposition. I.e., suppose we observe a sample from either $P$ or $Q$ and we have a prior on these two possibilities, what is the posterior distribution on possibilities?
The following gives such a result. However, it introduces an event $E_{P,Q}$. Intuitively, when $E_{P,Q}(Y)=1$, then we get the result we would get under pure DP. But $E_{P,Q}(Y)=0$ with probability $\delta$, in which case things can fail arbitrarily.

\textcite[Lemma 3.4]{kasiviswanathan2014semantics} provide a similar result. Ours improves the constant factors and is also stated slightly differently. 

\begin{lemma}\label{lem:adp_bayes}
    Let $P$ and $Q$ be probability distributions over $\mathcal{Y}$. Fix $\varepsilon, \delta \ge 0$. Suppose that, for all measurable $S \subset \mathcal{Y}$, we have $P(S) \le e^\varepsilon \cdot Q(S) + \delta$ and $Q(S) \le e^\varepsilon P(S) + \delta$.
    
    Then there exists a randomized function $E_{P,Q} : \mathcal{Y} \to \{0,1\}$ with the following properties.
    
    Fix $p \in [0,1]$ and suppose $X \gets \mathsf{Bernoulli}(p)$. If $X=1$, sample $Y \gets P$; and, if $X=0$, sample $Y \gets Q$.
    Then, %
    for all $y \in \mathcal{Y}$, we have \[ \pr{X \gets \mathsf{Bernoulli}(p) \atop Y \gets X P + (1-X) Q}{X=1 \wedge E_{P,Q}(Y)=1 | Y=y} \le \frac{p}{p+(1-p)e^{-\varepsilon}}.\]
    Furthermore,  \[\ex{Y \gets P}{E_{P,Q}(Y)} \ge 1-\delta ~~~~ \text{ and } ~~~~ \ex{Y \gets Q}{E_{P,Q}(Y)} \ge 1-\delta.\]
\end{lemma}
\begin{proof}
    We apply the decomposition from Lemma \ref{lem:kov}:
    There exist distributions $P'$, $Q'$, $P''$, and $Q''$ over $\mathcal{Y}$ and $\delta' \in [0,\delta]$ such that
    \begin{align*}
        P &= (1-\delta') P' + \delta' P''  ,\\
        Q &= (1-\delta') Q' + \delta' Q''  ,
    \end{align*}
    and, for all $y \in \mathcal{Y}$, $e^{-\varepsilon} P'(y) \le Q'(y) \le e^\varepsilon P'(y)$ and $P''(y)>0 \implies  P(y) \ge Q(y)$ and $Q''(y)>0 \implies  P(y) \le Q(y)$. (Here $P(\cdot)$ denotes the Radon-Nikodym derivative of the distribution $P$ with respect to some appropriate base measure and similarly for the other distributions.)

    We define $E_{P,Q} : \mathcal{Y} \to \{0,1\}$ by \[\pr{}{E_{P,Q}(y)=1} = (1-\delta') \cdot \frac{P'(y)}{P(y)} = 1 - \delta' \cdot \frac{P''(y)}{P(y)}.\]%
    Clearly, $\ex{Y \gets P, E_{P,Q}}{E_{P,Q}(Y)} = \int_{\mathcal{Y}} P(y) \pr{}{E_{P,Q}(Y)=1} \mathrm{d}y = \int_{\mathcal{Y}} (1-\delta) P'(y) \mathrm{d}y = 1-\delta' \ge 1-\delta$.
    Also
    \begin{align*}
        \ex{Y \gets Q}{E_{P,Q}(Y)} 
        &= 1 - \delta' \ex{Y \gets Q}{\frac{P''(y)}{P(y)}}\\
        &= 1 - \delta' \int_{\mathcal{Y}} \frac{Q(y)}{P(y)} \cdot P''(y) \mathrm{d}y\\
        &\ge 1 - \delta' \int_{\mathcal{Y}} P''(y) \mathrm{d}y\\
        &= 1-\delta' \ge 1-\delta,
    \end{align*}
    since $P''(y)>0 \implies P(y) \ge Q(y)$.
    For any $y \in \mathcal{Y}$, we have
    \begin{align*}
        & \pr{}{X=1 \wedge E_{P,Q}(Y)=1 | Y=y} \\
        &= \pr{}{X=1 | Y=y} \cdot \pr{}{E_{P,Q}(y)=1} \\
        &= \frac{\pr{}{Y=y | X=1} \cdot \pr{}{X=1}}{\pr{}{Y=y}} \cdot  \pr{}{E_{P,Q}(y)=1}  \\
        &= \frac{P(y) \cdot p}{pP(y)+(1-p)Q(y)} \cdot  \pr{}{E_{P,Q}(y)=1}  \\
        &= \frac{p(1-\delta')P'(y) + p\delta' P''(y)}{p(1-\delta')P'(y) + p\delta' P''(y) + (1-p)(1-\delta')Q'(y) + (1-p)\delta' Q''(y)} \cdot  \pr{}{E_{P,Q}(y)=1}  \\
        &= \frac{p + p\frac{\delta' P''(y)}{(1-\delta')P'(y)}}{p + p\frac{\delta' P''(y)}{(1-\delta')P'(y)} + (1-p)\frac{Q'(y)}{P'(y)} + (1-p)\frac{\delta' Q''(y)}{(1-\delta')P'(y)}} \cdot  \pr{}{E_{P,Q}(y)=1}  \\
        &= \frac{p + p\frac{\delta'}{1-\delta'}\frac{P''(y)}{P'(y)}}{p + (1-p)\frac{Q'(y)}{P'(y)} + \frac{\delta'}{1-\delta'} \cdot \frac{p P''(y) + (1-p) Q''(y)}{P'(y)}} \cdot  \pr{}{E_{P,Q}(y)=1}  \\
        &\le \frac{p + p\frac{\delta'}{1-\delta'}\frac{P''(y)}{P'(y)}}{p + (1-p)e^{-\varepsilon}+ 0} \cdot  \pr{}{E_{P,Q}(y)=1}  \\
        &= \frac{p}{p + (1-p)e^{-\varepsilon}} \cdot \left( 1 + \frac{\delta'}{1-\delta'}\frac{P''(y)}{P'(y)} \right) \cdot  \pr{}{E_{P,Q}(y)=1}  \\
        &= \frac{p}{p + (1-p)e^{-\varepsilon}} \cdot \left( \frac{P(y)}{(1-\delta')P'(y)} \right) \cdot  \pr{}{E_{P,Q}(y)=1}  \\
        &= \frac{p}{p + (1-p)e^{-\varepsilon}}.
    \end{align*}
\end{proof}

Now we can prove an analog of Proposition \ref{prop:pure} for the $(\varepsilon,\delta)$-DP setting.

\begin{proposition}[General Form of Main Result]\label{prop:approx}
    Let $M : \{-1,+1\}^m \to [-1,+1]^m$ satisfy $(\varepsilon,\delta)$-DP.
    Let $S \in \{-1,+1\}^m$ be $m$ independent samples from $2\mathsf{Bernoulli}(p)\!-\!1$ -- i.e., $\pr{}{S_i=1}=p$ independently for each $i \in [m]$. Let $T=M(S) \in [-1,+1]^m$.
    Then, for all $v \in \mathbb{R}$ and all $t \in [-1,+1]^m$,
    \[\prc{S \gets (2\mathsf{Bernoulli}(p)-1)^m, \atop T \gets M(S)}{\sum_i^m \max\{ 0 , T_i \cdot S_i\} \ge  v }{ T=t} \le\] \[\pr{\check{S}^+ \gets \mathsf{Bernoulli}\left( \frac{p \cdot e^\varepsilon}{p \cdot e^\varepsilon + 1-p} \right)^m \atop \check{S}^- \gets \mathsf{Bernoulli}\left( \frac{(1-p) \cdot e^\varepsilon}{(1-p) \cdot e^\varepsilon + p} \right)^m, F}{F(t) + \sum_{i\in [m] : t_i>0} t_i \cdot \check{S}_i^+ + \sum_{i \in [m] : t_i<0} -t_i \cdot \check{S}_i^- \ge  v}, \]
    where $F : [-1,1]^m \to \{0,1,\cdots,m\}$ is independent from $\check{S}^+$ and $\check{S}^-$ and satisfies $\ex{T,F}{F(T)} \le 2m \cdot \delta$.
\end{proposition}
\begin{proof}
    For $i \in [m] \cup \{0\}$ and $s_{\le i} \in \{-1,1\}^{i}$, let $M(s_{\le i})$ denote the distribution on $[-1,1]^m$ obtained by conditioning $M(S)$ on $S_{\le i} = s_{\le i}$. We can express this as a convex combination: \[M(s_{\le i}) = \sum_{s_{>i} \in \{-1,1\}^{m-i}} M(s_{\le i}, s_{>i}) \cdot \pr{S_{>i} \gets (2\mathsf{Bernoulli}(p)-1)^{m-i}}{S_{>i} = s_{>i}}.\]
    For distributions $P$ and $Q$ on $[-1,1]^m$, let $E_{P,Q} : [-1,1]^m \to \{0,1\}$ be the randomized function promised by Lemma \ref{lem:adp_bayes}. In our analysis, the internal randomness of $E_{P,Q}$ is independent from everything else -- i.e., the only dependence is induced by its input.
    Specifically, for all $i \in [m]$, all $s_{< i} \in \{-1,1\}^{i-1}$, and all $t \in [-1,1]^m$, we have
    \begin{align*}
        \prc{S \gets (2\mathsf{Bernoulli}(p)-1)^n, \atop T \gets M(S), E}{S_i = 1 \wedge E_{M(s_{<i},1),M(s_{<i},-1)}(T)=1 }{ S_{<i}=s_{<i}, T=t} &\le \frac{p \cdot e^\varepsilon}{p \cdot e^\varepsilon + 1-p},\\
        \exc{S \gets (2\mathsf{Bernoulli}(p)-1)^n, \atop T \gets M(S), E}{E_{M(s_{<i},1),M(s_{<i},-1)}(T)}{S_{\le i} = (s_{<i},1)} &\ge 1 - \delta,\\
        \exc{S \gets (2\mathsf{Bernoulli}(p)-1)^n, \atop T \gets M(S), E}{E_{M(s_{<i},1),M(s_{<i},-1)}(T)}{S_{\le i} = (s_{<i},-1)} &\ge 1 - \delta.
    \end{align*}
    Symmetrically, for all $i \in [m]$, all $s_{< i} \in \{-1,1\}^{i-1}$, and all $t \in [-1,1]^m$, we have 
    \begin{align*}
        \prc{S \gets (2\mathsf{Bernoulli}(p)-1)^n, \atop T \gets M(S), E}{S_i = -1 \wedge E_{M(s_{<i},-1),M(s_{<i},1)}(T)=1 }{ S_{<i}=s_{<i}, T=t} &\le \frac{(1-p) \cdot e^\varepsilon}{(1-p) \cdot e^\varepsilon + p},\\
        \exc{S \gets (2\mathsf{Bernoulli}(p)-1)^n, \atop T \gets M(S), E}{E_{M(s_{<i},-1),M(s_{<i},1)}(T)}{S_{\le i} = (s_{<i},-1)} &\ge 1 - \delta,\\
        \exc{S \gets (2\mathsf{Bernoulli}(p)-1)^n, \atop T \gets M(S), E}{E_{M(s_{<i},-1),M(s_{<i},1)}(T)}{S_{\le i} = (s_{<i},1)} &\ge 1 - \delta .
    \end{align*}
    For simplicity, we define a symmetric event: $E_P^Q(y) = E_Q^P(y) := E_{P,Q}(y) \cdot E_{Q,P}(y)$, where the internal randomnesses are again independent.
    Combining these, we have, for all $i \in [m]$, all $s_{< i} \in \{-1,1\}^{i-1}$, and all $t \in [-1,1]^m$, \[\prc{\!\!S \gets (2\mathsf{Bernoulli}(p)\!-\!1)^n,\!\! \atop T \gets M(S), E}{S_i \!=\! \sign(T_i) \!\wedge\! E_{M(s_{<i},1)}^{M(s_{<i},-1)}(T)\!=\!1 }{ T\!=\!t , S_{<i}\!=\!s_{<i}} \!\le\! \left\{ \!\!\begin{array}{cl} \frac{p \cdot e^\varepsilon}{p \cdot e^\varepsilon + 1-p} \!&\! \text{ if } t_i >0 \\ \frac{(1-p) \cdot e^\varepsilon}{(1-p) \cdot e^\varepsilon + p} \!&\! \text{ if } t_i < 0 \end{array} \!\!\right\}\]
    and, for $b \in \{-1,1\}$, we have \[        \exc{S \gets (2\mathsf{Bernoulli}(p)-1)^n, \atop T \gets M(S), E}{E_{M(s_{<i},-1)}^{M(s_{<i},1)}(T)}{S_{\le i} = (s_{<i},b)} \ge 1 - 2\delta .\]
    
    For $k \in [m]$, $s \in \{-1,1\}^m$, and $t \in [-1,1]^m$, define \[\widetilde{W}_k(s,t) := \sum_{i \in [k]} \max\{0, t_i \cdot s_i \} \cdot E_{M(s_{<i},1)}^{M(s_{<i},-1)}(t) = \sum_{i \in [k]} |t_i| \cdot \mathbb{I}[s_i = \sign(t_i) \wedge E_{M(s_{<i},1)}^{M(s_{<i},-1)}(t)=1]\] and \[\check{W}_k(t) = \sum_{i \in [k]} \check{S}_i(t) \cdot |t_i|,\] where, for each $i \in [k]$ independently, $\check{S}(t)_i \gets \mathsf{Bernoulli}(\frac{p \cdot e^\varepsilon}{p \cdot e^\varepsilon + 1-p})$ if $t_i > 0$ and $\check{S}(t)_i \gets \mathsf{Bernoulli}(\frac{(1-p) \cdot e^\varepsilon}{(1-p) \cdot e^\varepsilon + p})$ if $t_i < 0$.
    
    By induction and Lemma \ref{lem:sumdominance}, for any $k \in [m]$ and $t \in [-1,1]^m$, the conditional distribution $(\widetilde{W}_k(S,t)|M(S)=t)$ where $S \gets (2\mathsf{Bernoulli}(p)-1)^m$ is stochastically dominated by $\check{W}_k(t)$.
    
    For $s \in \{-1,1\}^m$ and $t \in [-1,1]^m$, define \[F(s,t) := \sum_i^m \mathbb{I}\left[ E_{M(s_{<i},1)}^{M(s_{<i},-1)}(t)=0 \right],\] so that \[W_m(s,t) := \sum_{i \in [m]} \max\{ 0 , t_i \cdot s_i\} \le \widetilde{W}_m(s,t) + F(s,t). \] Since the conditional distribution $(\widetilde{W}_k(S,t)|M(S)=t)$ where $S \gets (2\mathsf{Bernoulli}(p)-1)^m$ is stochastically dominated by $\check{W}_k(t)$, $W_m$ is stochastically dominated by the convolution $\check{W}_m(T) + F(S,T)$. 
    
    Finally $F(s,t)$ is supported on $\{0,1,\cdots,m\}$ and
    \begin{align*}
        \ex{}{F(s,t)} &= \sum_i^m \pr{}{E_{M(s_{<i},1)}^{M(s_{<i},-1)}(T)=0 } \le 2m \cdot \delta .
    \end{align*}
    Since $\check{W}_m(T)$ does not depend on $S$, the input $S$ does not contribute to the dependence between $F(S,T)$ and $\check{W}_m(T)$, so we can elide this input in the statement -- i.e., $F(T)=F(S,T)$ for $S$ drawn from an appropriate distribution.
\end{proof}

Proposition \ref{prop:approx} is rather unwieldy. It can be simplified by setting $p=\frac12$ and identifying the optimal distribution $F(T)$, which yields Theorem \ref{thm:main}.

\begin{proof}[Proof of Theorem \ref{thm:main}.]
    Let $M : \{-1,1\}^m \to [-1,1]^m$ satisfy $(\varepsilon,\delta)$-DP.
    Let $S \in \{-1,1\}^m$ be uniformly random. Let $T=M(S) \in [-1,1]^m$.
    Setting $p=\frac12$ in Proposition \ref{prop:approx} and averaging over $T$, we have, for all $v \in \mathbb{Z}$,
    \[\pr{S \gets \{-1,1\}^m \atop T \gets M(S)}{\sum_i^m \max\{ 0 , T_i \cdot S_i \} \ge  v } \le \pr{S \gets \{-1,1\}^m , T \gets M(S), \atop \check{S} \gets \mathsf{Bernoulli}\left( \frac{e^\varepsilon}{e^\varepsilon + 1} \right)^m, F}{ F(T) + \sum_i^m \check{S}_i \cdot |T_i| \ge  v}, \]
    where $F$ is arbitrary -- but independent from $\check{S}$ -- except for the constraints that $F(T)$ is supported on $\{0,1,\cdots,m\}$ and $\ex{}{F(T)} \le 2m \cdot \delta$.
    
    Given these constraints, we can formulate finding the optimal distribution $F(t)$ for a given $t \in [-1,1]^m$ and $v \in \mathbb{R}$ as a linear program:
    \begin{align*}
        \text{maximize} ~~&~~ \pr{\check{W}, F}{\check{W}(t) + F(t) \ge  v} = \sum_{i=0}^m \pr{F}{F(t)=i} \cdot \pr{\check{W} }{\check{W}(t) \ge  v-i} \\
        \text{subject to} ~~&~~ \ex{F}{F(t)} = \sum_{i=0}^m \pr{F}{F(t)=i} \cdot i \le 2m \cdot \delta ,\\
        &~~ \sum_{i=0}^m \pr{F}{F(t)=i}=1, \text{ and }\\
        &~~ \pr{F}{F(t)=i} \ge 0 ~~\forall i \in \{0,1,\cdots,m\} ,
    \end{align*}
    where $\check{W}(t) := \sum_i^m \check{S}_i |t_i|$ for $\check{S} \gets \mathsf{Bernoulli}\left(\frac{e^\varepsilon}{e^\varepsilon+1}\right)^m$.
    
    By strong duality, the linear program above has the same value as its dual:
    \begin{align*}
        \text{minimize} ~~&~~ 2m\delta \alpha + \beta \\
        \text{subject to} ~~&~~ \alpha \cdot i + \beta \ge \pr{\check{W}}{\check{W}(t) \ge  v-i} ~~\forall i \in \{0,1,\cdots,m\}, \\
        &~~ \alpha \ge 0.
    \end{align*}
    Any feasible solution to the dual gives an upper bound on the primal. So, in particular, we can use the solution given by
    \begin{align*}
        \beta &= \pr{\check{W}^*}{\check{W}^* \ge  v},\\
        \alpha &= \max \left(\{ 0 \} \cup \left\{ \frac{1}{i} \left( \pr{\check{W}^*}{\check{W}^* \ge  v-i} - \beta \right) : i \in \{1,2,\cdots,m\} \right\}\right),
    \end{align*}
    where $\check{W}^*$ is a distribution on $\mathbb{R}$ that satisfies $\pr{\check{W}^*}{\check{W}^* \ge v-i} \ge \pr{\check{W}}{\check{W}(t) \ge v-i}$ for all $i \in \{0,1,\cdots,m\}$ and all $t$ in the support of $T$.
\end{proof}

Theorem \ref{thm:main} gives a worst-case bound in terms of $T$. Specifically, $\check{W}^*$ must uniformly bound $\check{W}(t)$ for all $t$ in the support of $T$.
Proposition \ref{prop:approx} is more general than this. 
Thus we give another corollary that allows us to have the bound adjust to $T$.
In particular, this result allows the auditing procedure (Algorithm \ref{alg:generic-audit} or \ref{alg:audit}) to dynamically choose the number of guesses $r=k_++k_-$.

\begin{corollary}[Variant of Main Result]\label{cor:adap_T}
    Let $M : \{-1,+1\}^m \to [-1,+1]^m$ satisfy $(\varepsilon,\delta)$-DP. Let $S \in \{-1,+1\}^m$ be uniformly random. Let $T=M(S) \in [-1,+1]^m$. Then, for all $\gamma \in [0,1]$ and $\tau>0$,
    \[\pr{S \gets \{-1,+1\}^m \atop T \gets M(S)}{\sum_i^m \max\{ 0 , T_i \cdot S_i\} \ge g_{m,\varepsilon}(T,\gamma)+\tau} \le \gamma + \frac{2m\delta}{\tau},\]
    where $g_{m,\varepsilon} : [-1,+1]^m \times [0,1] \to \mathbb{R}$ is an arbitrary function satisfying \[\forall t \in [-1,1]^m ~ \forall \gamma \in [0,1] ~~\pr{\check{S} \gets \mathsf{Bernoulli}\left(\frac{e^\varepsilon}{e^\varepsilon+1}\right)^m}{\sum_i^m |t_i| \cdot \check{S}_i \ge g_{m,\varepsilon}(t,\gamma)} \le \gamma.\]
\end{corollary}
\begin{proof}
    Setting $p=\frac12$ in Proposition \ref{prop:approx} yields
    \begin{align*}
        \forall v \in \mathbb{R} ~ \forall t \in [-1,+1]^m ~~~ &\prc{S \gets \{-1,+1\}^m \atop T \gets M(S)}{\sum_i^m \max\{ 0 , T_i \cdot S_i \} \ge v}{T=t} \\
        &\le \pr{\check{S} \gets \mathsf{Bernoulli}\left(\frac{e^\varepsilon}{e^\varepsilon+1}\right),F}{F(t) + \sum_i^m |t_i| \cdot \check{S}_i \ge v},
    \end{align*}
    where $F : [-1,1]^m \to \{0,1,\cdots,m\}$ satisfies $\ex{S \gets \{-1,+1\}^m \atop T \gets M(S),F}{F(T)}\le 2m\delta$.

    By a union bound and Markov's inequality, we have, for all $t \in [-1,1]^m$, all $\gamma \in [0,1]$, and all $\tau>0$,
    \begin{align*}
        &\pr{\check{S} \gets \mathsf{Bernoulli}\left(\frac{e^\varepsilon}{e^\varepsilon+1}\right),F}{F(t) + \sum_i^m |t_i| \cdot \check{S}_i \ge g_{m,\varepsilon}(t,\gamma)+\tau} \\
        &\le \pr{\check{S} \gets \mathsf{Bernoulli}\left(\frac{e^\varepsilon}{e^\varepsilon+1}\right),F}{\tau + \sum_i^m |t_i| \cdot \check{S}_i \ge g_{m,\varepsilon}(t,\gamma)+\tau} + \pr{F}{F(t) > \tau} \\
        &\le \pr{\check{S} \gets \mathsf{Bernoulli}\left(\frac{e^\varepsilon}{e^\varepsilon+1}\right),F}{\sum_i^m |t_i| \cdot \check{S}_i \ge g_{m,\varepsilon}(t,\gamma)} + \frac{\ex{F}{F(t)}}{\tau} \\
        &\le \gamma + \frac{\ex{F}{F(t)}}{\tau}.
    \end{align*}
    
    We combine inequalities, set $v=g_{m,\varepsilon}(t,\gamma)+\tau$, and average over $T$ to obtain
    \begin{align*}
        \forall \gamma \in [0,1] ~ \forall \tau>0 ~~~
        &\pr{S \gets \{-1,+1\}^m \atop T \gets M(S)}{\sum_i^m \max\{ 0 , T_i \cdot S_i \} \ge g_{m,\varepsilon}(T,\gamma)+\tau} \\
        &\le \gamma + \frac{\ex{S \gets \{-1,+1\}^m \atop T \gets M(S), F}{F(T)}}{\tau} \\
        &\le \gamma + \frac{2m\delta}{\tau}.
    \end{align*}
\end{proof}

\begin{algorithm}[h]
    \caption{DP-SGD Auditor (Instantiation of Algorithm \ref{alg:generic-audit})}\label{alg:audit}
    \begin{footnotesize}
    \begin{algorithmic}[1]
        \State \textbf{Data:} $x \in \mathcal{X}^{n}$ consisting of $m$ auditing examples (a.k.a.~canaries) and $n-m$ non-auditing examples.
        \State \textbf{Parameters:} Number of examples to randomize $m$ for audit, number of positive $k_+$ and negative $k_-$ guesses \texttt{audit-type} (either \texttt{black-box} or \texttt{white-box}).
        \State For $i \in [m]$ sample $S_i \in \{-1,+1\}$ independently with $\ex{}{S_i}=0$. Set $S_i=1$ for all $i \in [n] \setminus [m]$.
        \State Split $x$ into $\xin \in \mathcal{X}^\nin$ and $\xout \in \mathcal{X}^\nout$ according to $S$, where $\nin+\nout=n$. Namely, if $S_i=1$, then $x_i$ is in $\xin$; and, if $S_i=-1$, then $x_i$ is in $\xout$.
        \State Run DP-SGD (Algorithm \ref{alg:dpsgd}) on input $\xin$ with appropriate parameters.
        \State Let $\ell$ be the number of iterations and let $f : \mathbb{R}^d \times \mathcal{X} \to \mathbb{R}$ be the loss.
        \State Let $w^0, \cdots, w^\ell \in \mathbb{R}^d$ be the output of DP-SGD.
        \If{$\text{\texttt{audit-type}}=\text{\texttt{black-box}}$}
            \State Define \textsc{Score}$(x_i, w^\ell) = \ell(w^0,x_i) - \ell(w^\ell,x_i)$ for all $i \in [m]$.
            \State Compute the vector of scores $Y = \left( \textsc{Score}(x_i,w^{\ell}) : i \in [m] \right) \in \mathbb{R}^m$.
        \ElsIf{$\text{\texttt{audit-type}}=\text{\texttt{white-box}}$}
            \Procedure{Score}{$x_*$, $w^1, \cdots, w^\ell$}
                \For{$t=1,\cdots,\ell$}
                    \State Compute $g^t = \nabla_{w_{t-1}} \ell(w_{t-1},x_*) \in \mathbb{R}^d$.
                    \State Clip $\hat{g}^t = \min\left\{1,\frac{c}{\|g^t\|_2}\right\} \cdot g^t \in \mathbb{R}^d$.
                    \State Let $v^t = \left\langle w^{t-1} - w^t , \hat{g}^t \right\rangle \in \mathbb{R}$.
                \EndFor
                \State Return $\sum_{t=1}^\ell v^t \in \mathbb{R}$.
            \EndProcedure
            \State Compute the vector of scores $Y = \left( \textsc{Score}(x_i,w^{0},w^1,\cdots,w^\ell) : i \in [m] \right) \in \mathbb{R}^m$.
        \EndIf
        \State Sort the scores $Y$. Let $T \in \{-1,0,+1\}^m$ be $+1$ for the largest $k_+$ scores and $-1$ for the smallest $k_-$ scores.\\(I.e., $T \in \{-1,0,+1\}^m$ maximizes $\sum_i^m T_i \cdot Y_i$ subject to $\sum_i^m |T_i| = k_++k_-$ and $\sum_i^m T_i = k_+-k_-$.)
        \State \textbf{Return:} The vector $S \in \{-1,+1\}^m$ indicating the true selection and the guesses $T \in \{-1,0,+1\}^m$.
    \end{algorithmic}
    \end{footnotesize}
\end{algorithm}

\section{Experiments}\label{sec:experiments}

\paragraph{Experiment Setup}
Our contributions are focused on improved analysis of an existing privacy attack, and are therefore orthogonal to the design of an attack. As a result, we rely on the experimental setup of the recent auditing procedure of \citet{nasr2023tight}.

We run DP-SGD on the CIFAR-10 dataset with Wide ResNet (WRN-16) \cite{zagoruyko2016wide}, we followed the experimental setup from \citetal{nasr2023tight}. Our experiments reach $~76\%$ test accuracy at $(\varepsilon=8, \delta=10^{-5})$-DP, which is comparable with the state-of-the-art~\cite{de2022unlocking}. Unless specified otherwise, all lower bounds are presented with $95\%$ confidence. Following \citetal{nasr2023tight}, we refer to the setting where the adversary has access to all intermediate steps as ``white-box'' and when the adversary can only see the last iteration as ``black-box.'' We experiment with both settings.  

Algorithm~\ref{alg:audit} summarizes our approach for auditing DP-SGD. The results are converted into lower bounds on the privacy parameters using Theorem \ref{thm:main} / Corollary \ref{cor:ternary}.

We also experiment with both the gradient and input attacks proposed by \citetal{nasr2023tight}. In particular, for the gradient attack we use the strongest attack they proposed -- the ``Dirac canary'' approach -- which  sets all gradients to zero except at a single random index. In our setting where we need to create multiple auditing examples (canaries) we make sure the indices selected in our experiments do not have any repetitions. 
To compute the score for gradient space attacks, we use the dot product between the gradient update and auditing gradient. When auditing in input space, we leverage two different types of injected examples as:
\begin{enumerate}
    \item \textbf{Mislabeled  example}: We select a random subset of the test set and randomly relabel them (ensuring the new label is not the same as the original label).
    \item \textbf{In-distribution example}: We select a random subset of the test set. 
\end{enumerate}

For input space audits, we use the loss of the input example as the score. In our experiments we report the attack with the highest lower bound.

 In our experiments, we evaluate different values of $k_+$ and $k_-$ and only report the highest auditing results. Since this is doing multiple hypothesis testing on the same data, we are reducing the confidence value of our results. However, this is commonly used in the previous works~\cite{zanella2022bayesian, maddock2022canife} and can be easily improved by using a different set of observations to select the parameters for the auditing and another set of the data for the auditing itself (see also Corollary \ref{cor:adap_T}).
\subsection{Gradient Space attacks}

\begin{figure}
    \centering
      \begin{minipage}[t]{0.45\textwidth}
    \centering
    \includegraphics[scale=0.5]{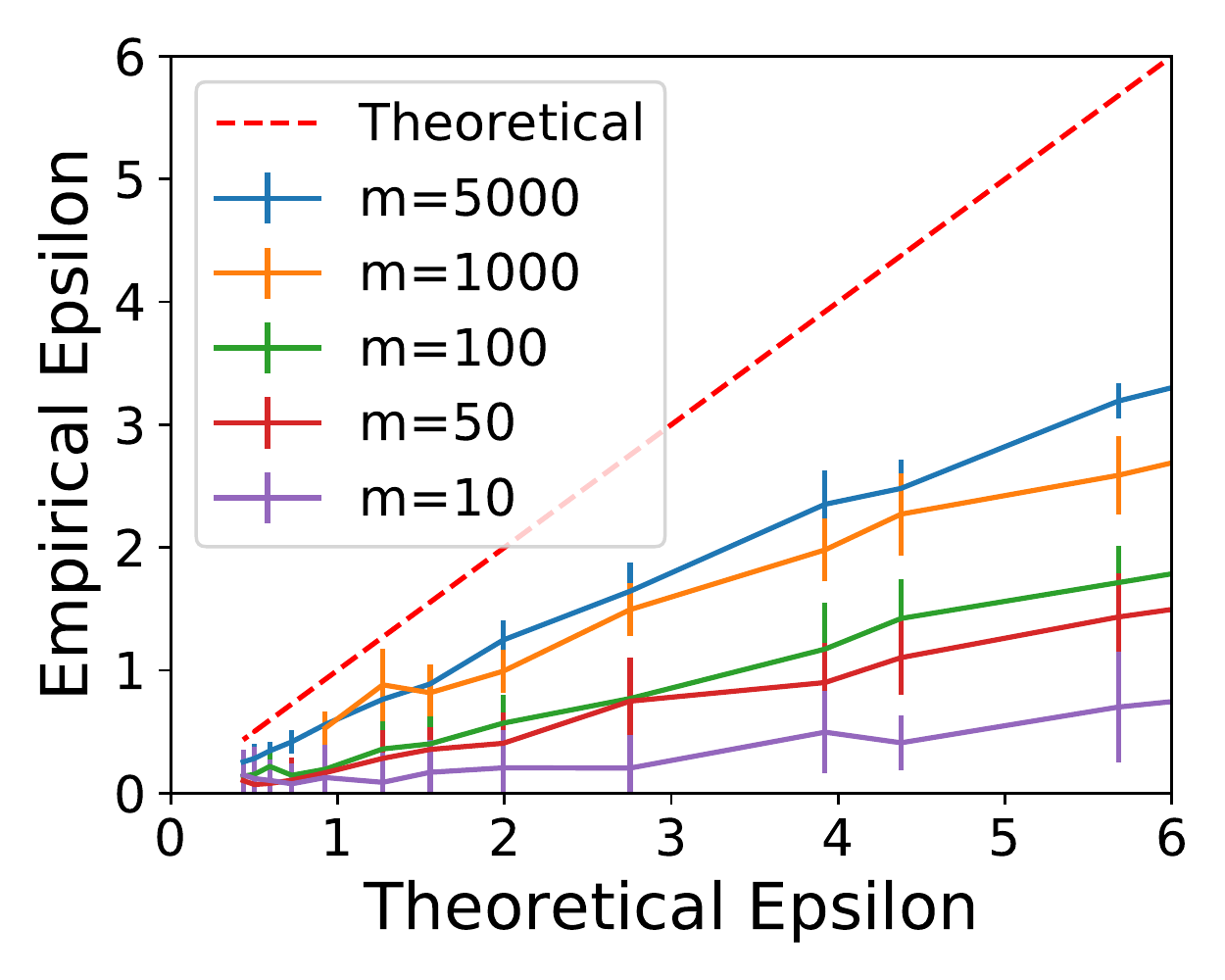}
    \caption{Effect of the number of auditing examples ($m$) in the white-box setting. By increasing the number of the auditing examples we are able to achieve tighter empirical lower bounds. }
    \label{fig:wb_diff_m}
\end{minipage}
\hfill
\begin{minipage}[t]{0.45\textwidth}
    \centering
    \includegraphics[scale=0.5]{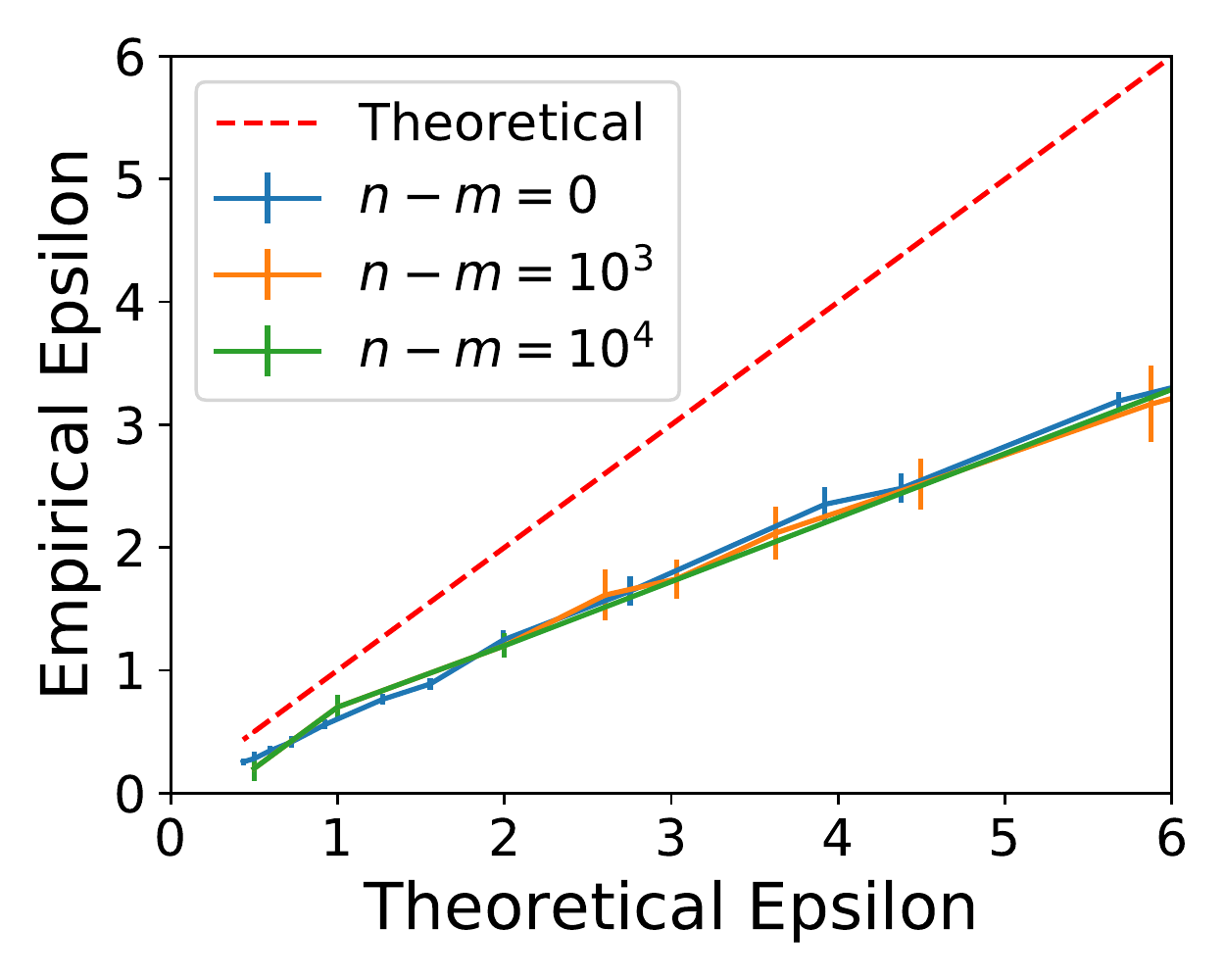}
    \caption{Effect of the number of additional examples ($n-m$) in the white-box setting. Importantly, adding additional examples does not impact the auditing results in the white-box setting. }
    \label{fig:wb_diff_nm}
\end{minipage}
\end{figure}

We start with the strongest attack: We assume white-box access -- i.e., the auditor sees all intermediate iterates of DP-SGD -- and that the auditor can insert examples with arbitrary gradients into the training procedure. First, we evaluate the effect of the number of the auditing example on the tightness. Figure~\ref{fig:wb_diff_m} demonstrates that as the number of examples increases, the auditing becomes tighter. However, the impact of the additional examples eventually diminishes. Intriguingly, adding more non-auditing training examples (resulting in a larger $n$ compared to $m$) does not seem to influence the tightness of the auditing, as depicted in Figure~\ref{fig:wb_diff_nm}. This can be primarily due to the fact that gradient attacks proposed in prior studies can generate near-worst-case datasets, irrespective of the presence of other data points. 

\begin{figure}[t]
  \centering
\begin{minipage}[t]{0.45\textwidth}

    \centering
    \includegraphics[scale=0.5]{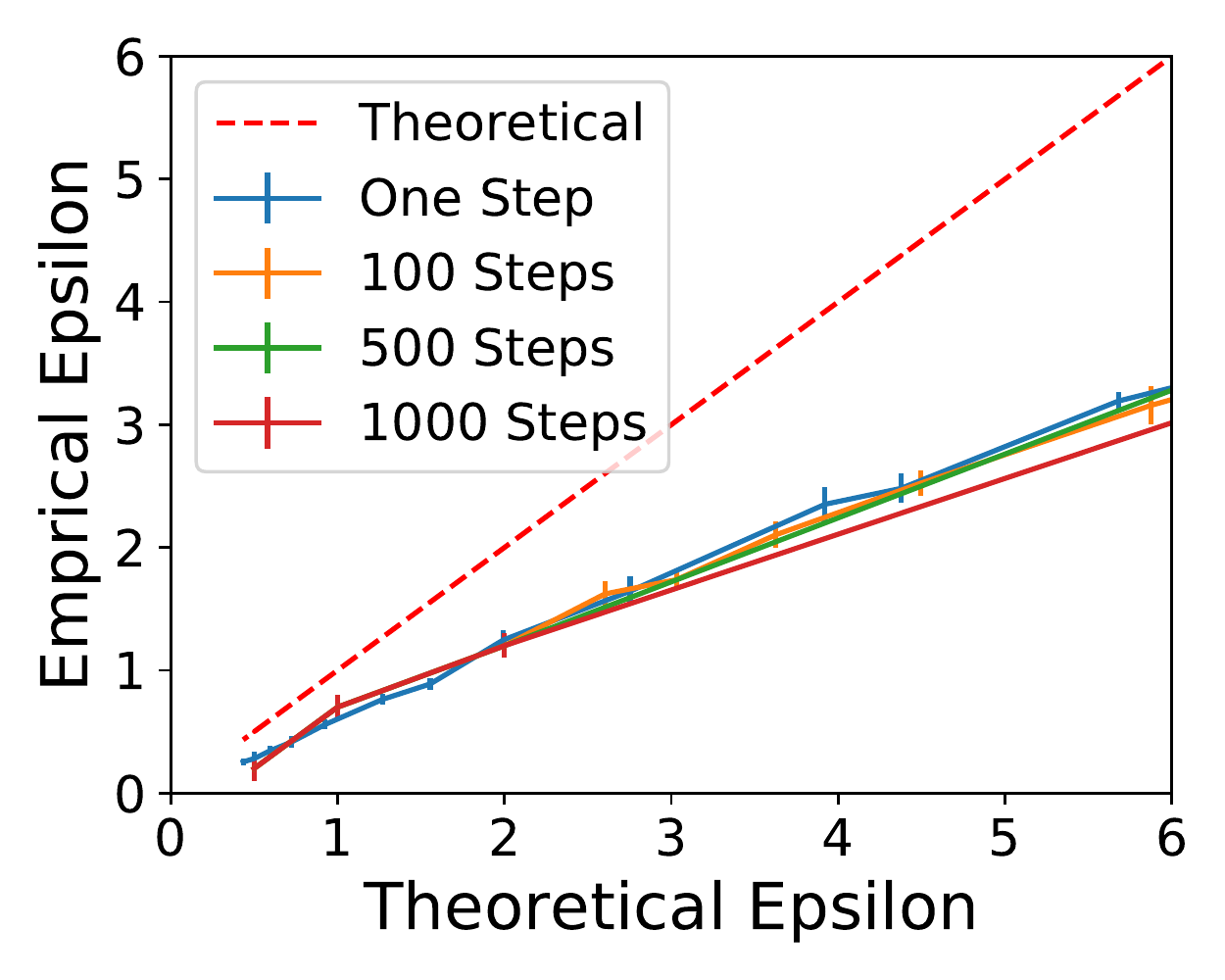}
    \caption{Effect of number of iterations in the white-box setting. Increasing the number of the steps (while keeping the same overall privacy by increasing the added noise) will not effect the auditing results.}
    \label{fig:wb_diff_iterations}
    \end{minipage}
\hfill
\begin{minipage}[t]{0.45\textwidth}
        \centering
    \includegraphics[scale=0.5]{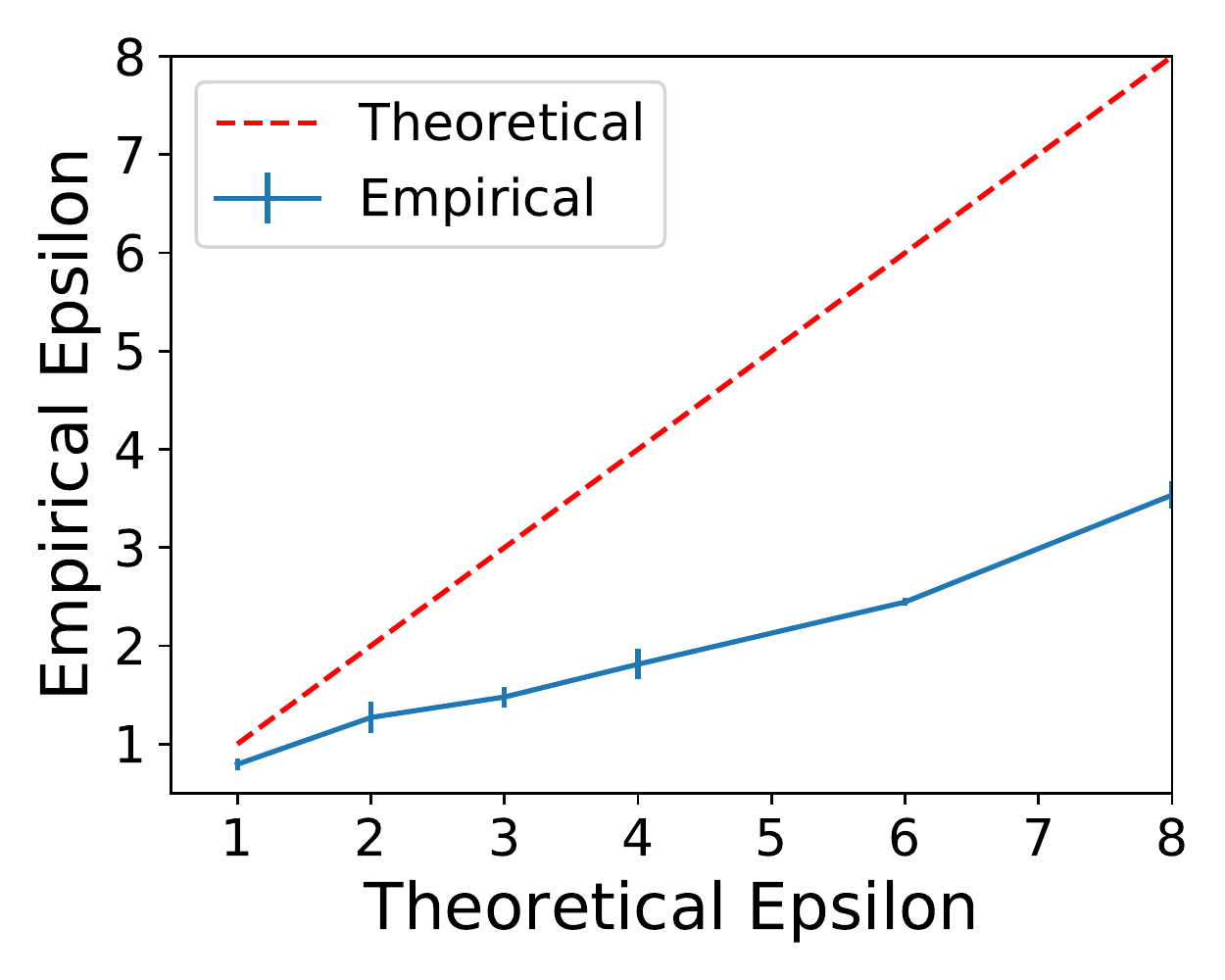}
    
    \caption{Auditing CIFAR10 SoTA in white-box setting using gradient attacks. Our auditing framework can achieve meaningful empirical privacy lower bounds for SoTA models. }
    \label{fig:wb_dpsgd_cifar}
\end{minipage}
\end{figure}

 Another parameter that might affect the auditing results is the number of iterations $\ell$ in the DP-SGD algorithm. As shown in Figure~\ref{fig:wb_diff_iterations} we compare the extreme setting of having one iteration to multiple iterations and we do not observe any significant difference in the auditing when auditing for the equivalent privacy guarantees (by increasing the noise). The results confirm the tightness of composition and that the number of iterations does not have significant effect on auditing in white-box setting.

Now we directly use the parameters used in the training CIFAR10 models. Figure~\ref{fig:wb_dpsgd_cifar} summarizes results for the CIFAR10 models. We used $m=5000$ and all of the training dataset from CIFAR10 ($n=50,000$) for the attack. We were able to achieve $76\%$ accuracy for $\varepsilon=8$ ($\delta=10^{-5}$, compared to $78\%$ when not auditing). We are able to achieve an empirical lower bound of $0.7,1.2,1.8,3.5$ for theoretical epsilon of $1,2,4,8$ respectively. While our results are not as tight as the prior works, we only require a single run of training which is not possible using the existing techniques. In the era of exponentially expanding machine learning models, the computational and financial costs of training these colossal architectures even once are significant. Expecting any individual or entity to shoulder the burden of training such models thousands of times for the sake of auditing or experimental purposes is both unrealistic and economically infeasible. Our method offers a unique advantage by facilitating the auditing of these models, allowing for an estimation of privacy leakage in a white-box setting without significantly affecting performance.

\subsection{Input Space Attacks}
Now we evaluate the effect of input space attacks in the black-box setting. In this attack, the auditor can only insert actual images into the training procedure and cannot control any of the aspects of the training. Then, the adversary can observe the final model as mentioned in Algorithm~\ref{alg:audit}. This is the weakest attack setting.

For simplicity we start with the setting where $m=n$; in other words, all of the examples used to train the model are randomly included or excluded and can be used for auditing.
Figure~\ref{fig:bb_diff_m} illustrates the result of this setting. As we see from the figure, unlike the white-box attack we do not observe a monotonic relationship between the number of auditing examples and the tightness of the auditing. Intuitively, when the number of auditing examples are low then we do not have enough observations to have high confidence lower bounds for epsilon. On the other hand, when the number of auditing examples are high, the model does not have enough capacity to ``memorize'' all of the auditing examples which reduces the tightness of the auditing. However, this can be improved by designing better black-box attacks which we reiterate in the next section. 

We also evaluate the effect of adding additional training data to the auditing in Figure~\ref{fig:bb_diff_nm}. We see that adding superfluous training data significantly reduces the effectiveness of auditing. The observed reduction in auditing effectiveness with the addition of more training data could be attributed to several factors. One interpretation could be that the theoretical privacy analysis in a black-box setting tends to be considerably more loose when the adversary is constrained to this setting. This could potentially result in an overestimation of the privacy bounds. Conversely, it is also plausible that the results are due to the weak black-box attacks and can be improved in the future.

\begin{figure}[t]
  \centering
\begin{minipage}[t]{0.45\textwidth}
        \centering
    \includegraphics[scale=0.5]{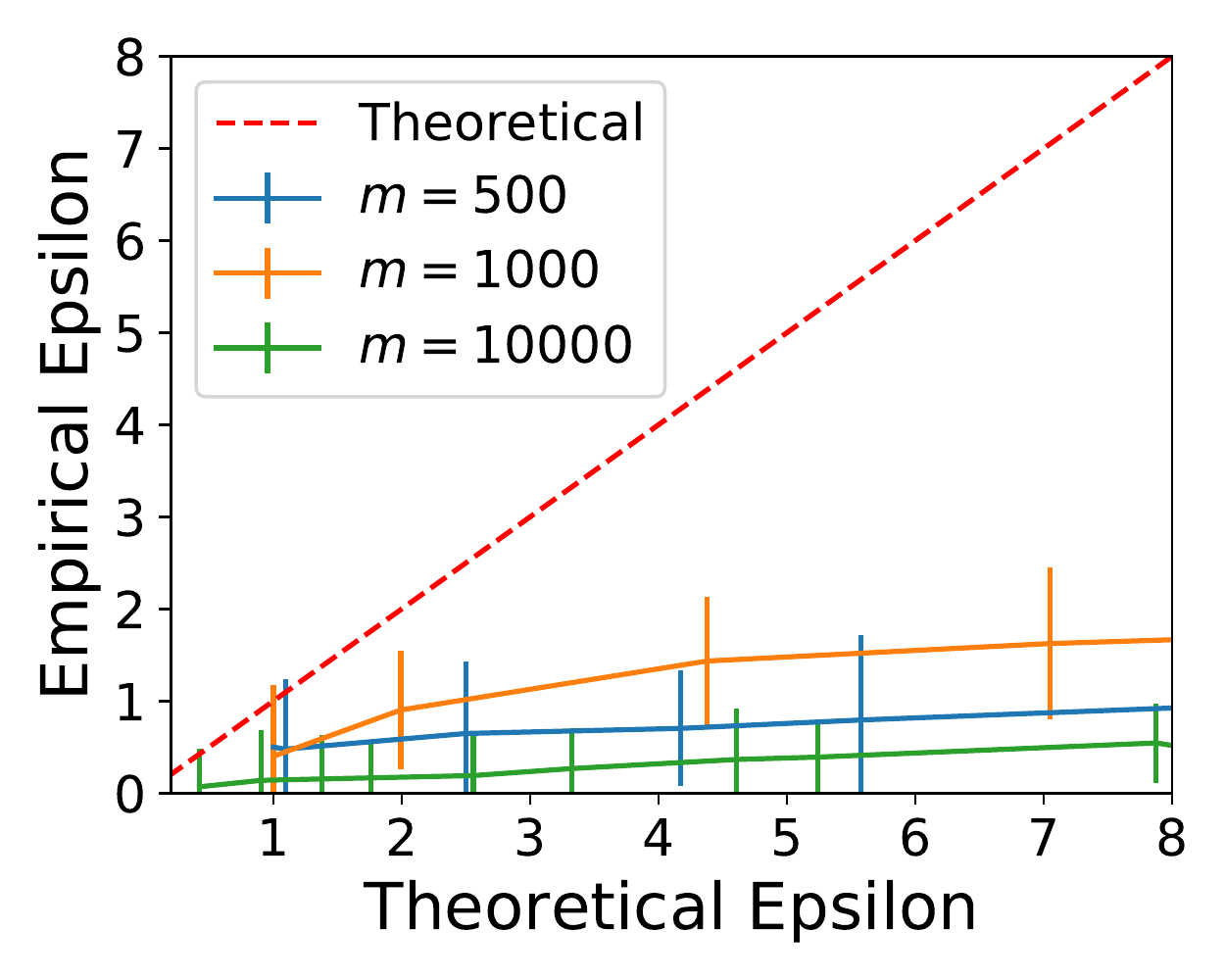}
    
    \caption{Effect of the number of auditing examples ($m$) in the black-box setting. Black-box auditing is very sensitive to the number of auditing examples.}
    \label{fig:bb_diff_m}
\end{minipage}
\hfill
\begin{minipage}[t]{0.45\textwidth}
    \centering
    \includegraphics[scale=0.5]{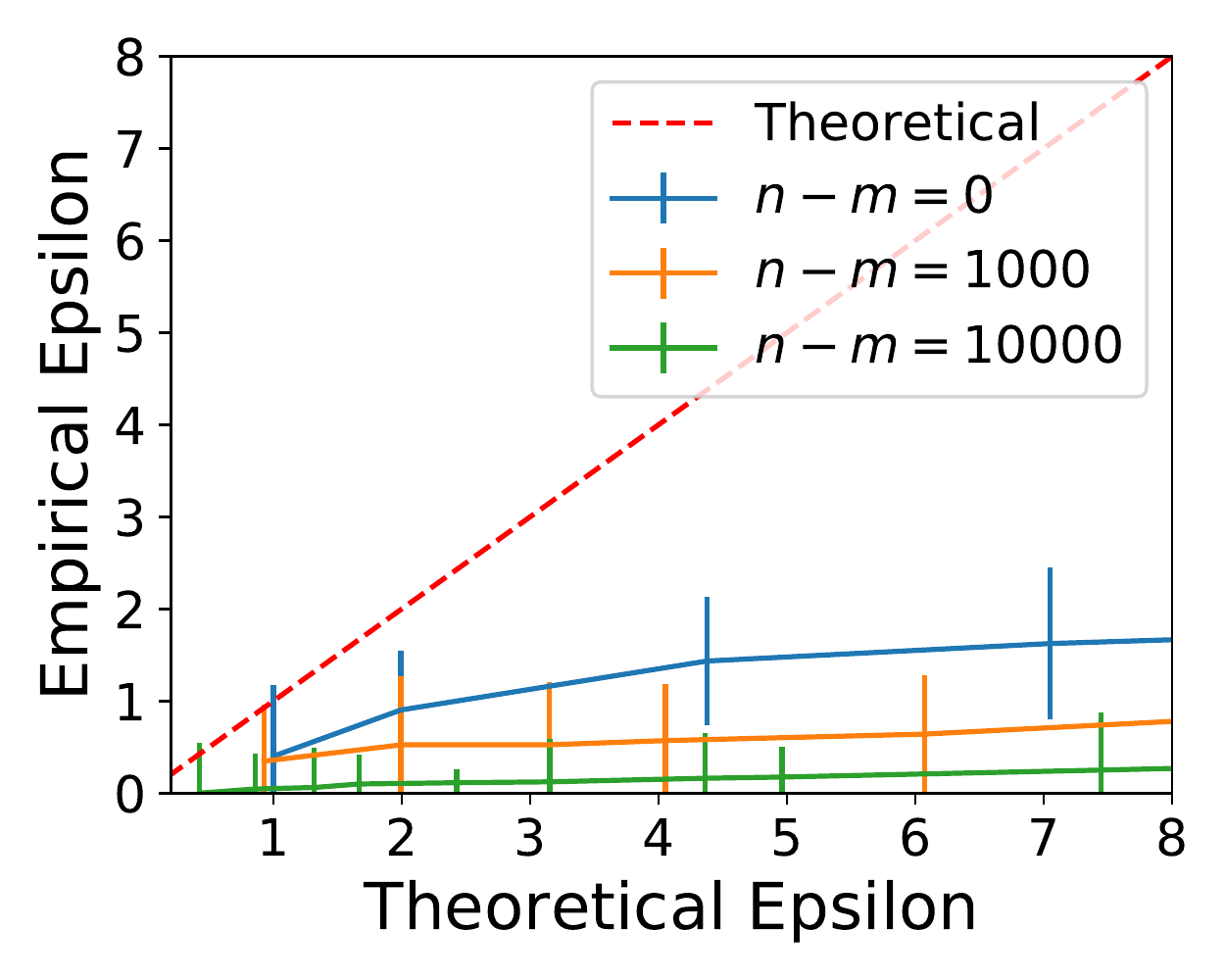}
    \caption{Effect of the number of additional example on auditing ($n-m$) in the black-box setting. By increasing the number of additional examples, the auditing results get significantly looser.}
    \label{fig:bb_diff_nm}
    \end{minipage}
\end{figure}

\section{Discussion}%

\begin{figure}[h]
    \centering
    \includegraphics[width=0.75\textwidth]{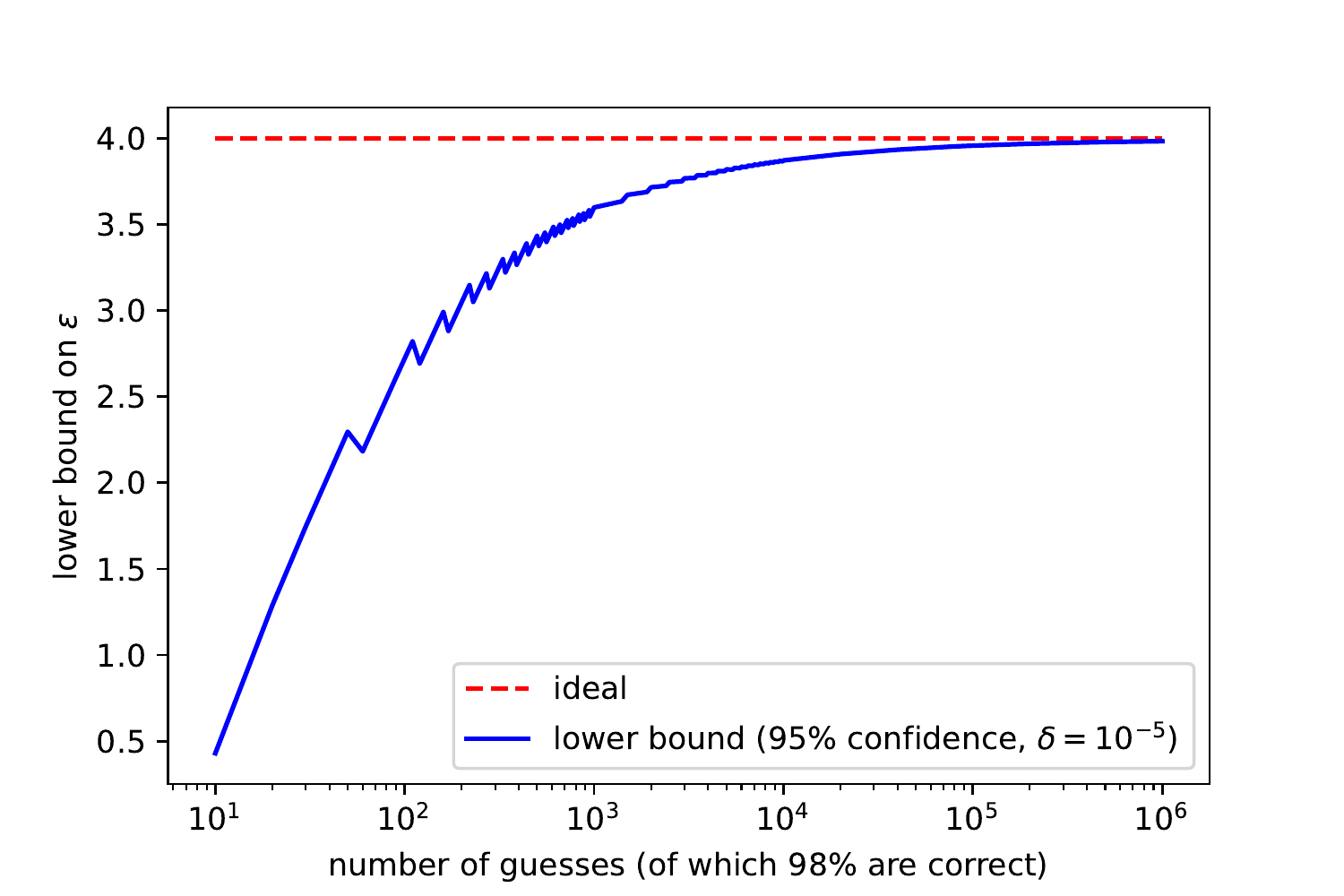}
    \caption{
        Comparison of upper and lower bounds for idealized setting with varying number of guesses. The fraction of correct guesses is always $\frac{e^\varepsilon}{e^\varepsilon+1}$ for $\varepsilon=4$ (i.e., 98.2\%).
    }
    \label{fig:guesses-eps-pure}
\end{figure}
\begin{figure}[h]
    \centering
    \includegraphics[width=0.75\textwidth]{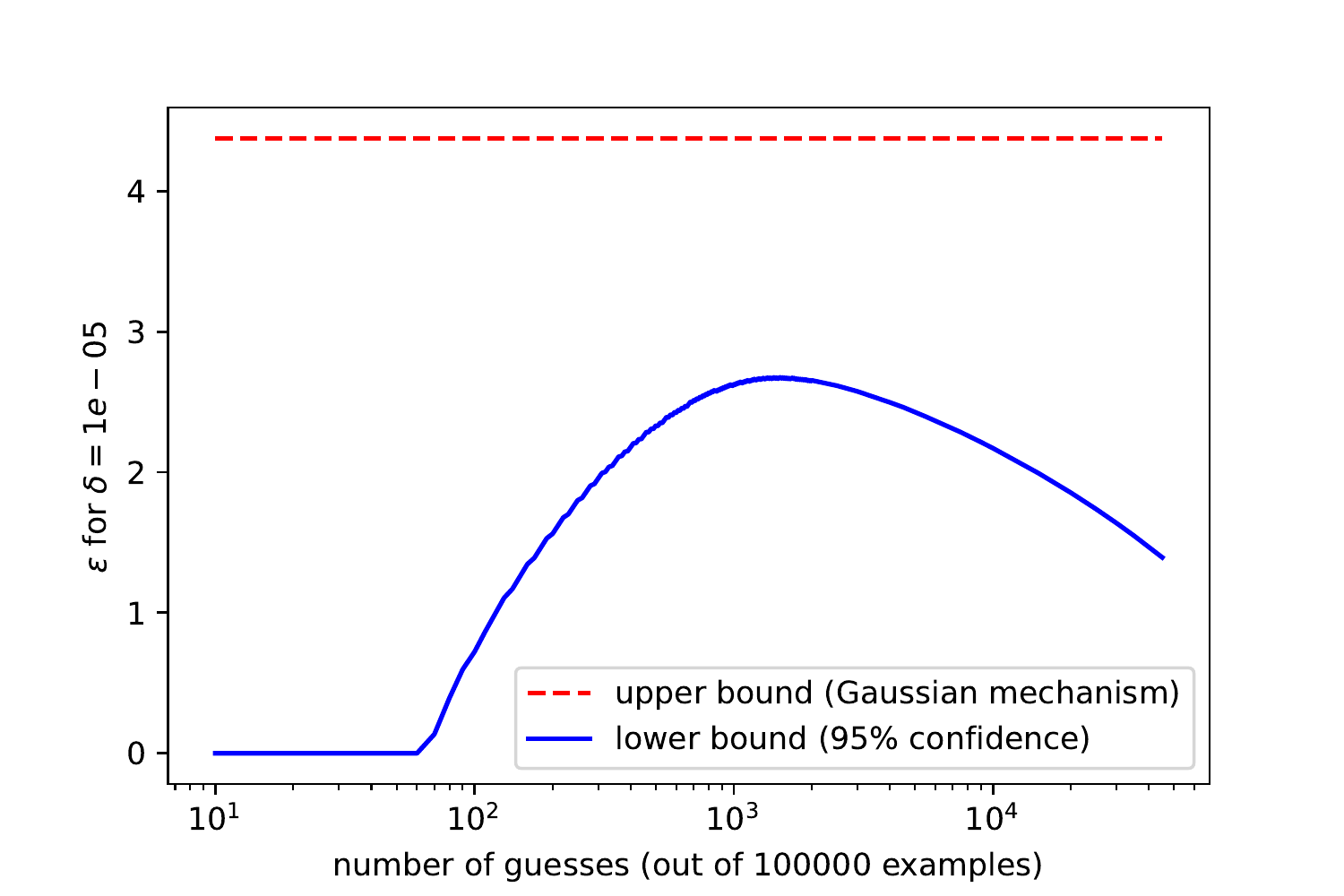}
    \caption{
        Comparison of upper and lower bounds for idealized setting with varying number of guesses. For each example $i \in [m]$, we release $S_i + \xi_i$, where $\xi_i \gets \mathcal{N}(0,4)$ and $S_i \in \{-1,+1\}$ is independently uniformly random and indicates whether the sample is included/excluded. For the upper bound, we compute the exact $(4.38,10^{-5})$-DP guarantee for the Gaussian mechanism (Lemma \ref{lem:gauss}). For the lower bound, we plot the bound of Theorem \ref{thm:main} with 95\% confidence for varying numbers of guesses $r$. We consider a total of $m=100,000$ randomized examples; we guess $T_i=+1$ for the largest $r/2$ scores and we guess $T_i=-1$ for the smallest $r/2$ scores; we guess $T_i=0$ for the remaining $m-r$ examples. 
        The number of correct guesses is set to $\left\lceil r \cdot \prc{}{S_i=+1}{S_i+\xi_i > c} \right\rceil$, where $c$ is a threshold such that $\pr{}{S_i+\xi_i > c} = \frac{r}{2m}$.
    }
    \label{fig:guesses-eps}
\end{figure}
\begin{figure}[h]
    \centering
    \includegraphics[width=0.75\textwidth]{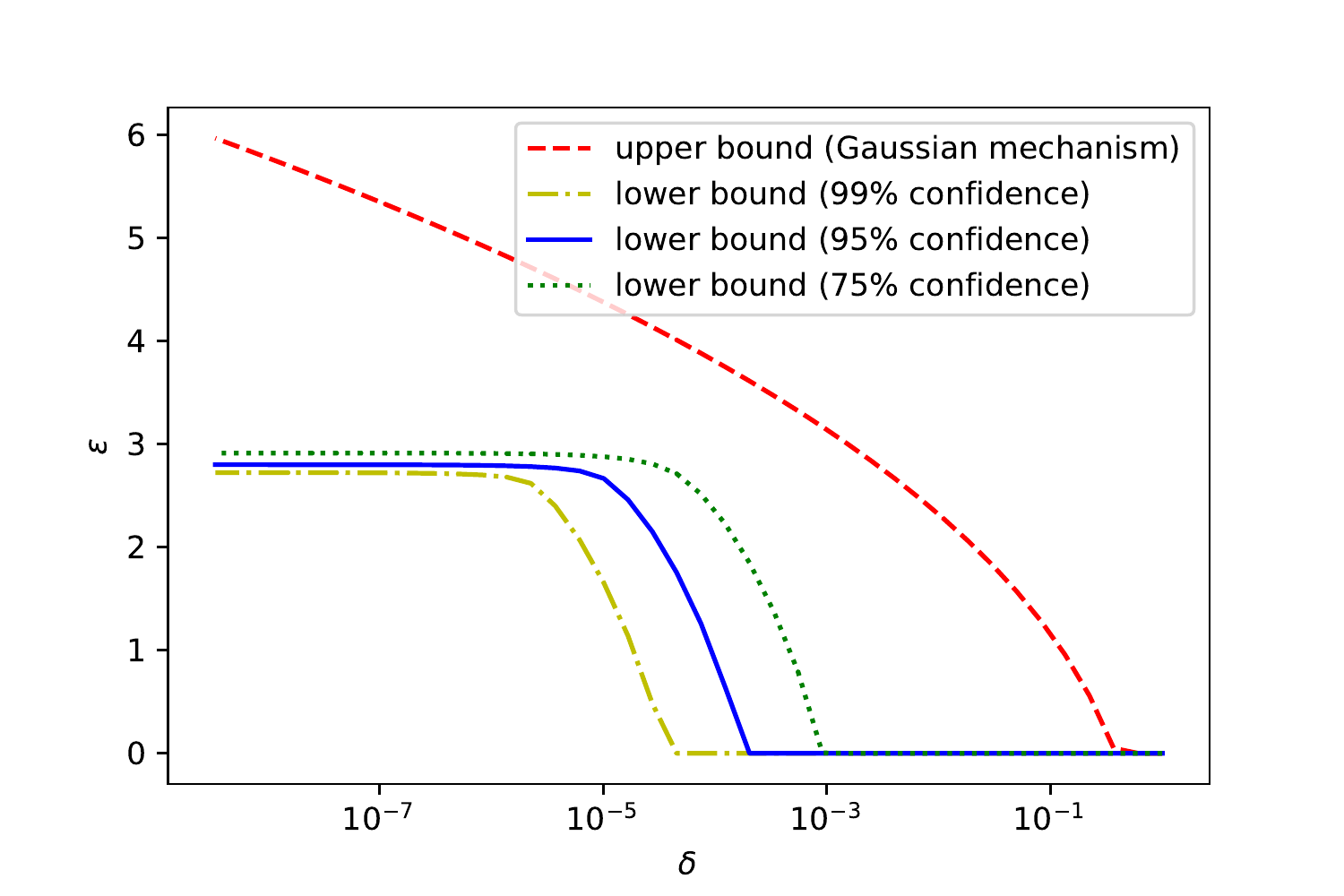}
    \caption{
        Comparison of upper and lower bounds for idealized setting with varying $\delta$. For each example $i \in [m]$, we release $S_i + \xi_i$, where $\xi_i \gets \mathcal{N}(0,4)$ and $S_i \in \{-1,+1\}$ is independently uniformly random and indicates whether the sample is included/excluded. For the upper bound, we compute the exact $(4.38,10^{-5})$-DP guarantee for the Gaussian mechanism (Lemma \ref{lem:gauss}). For the lower bound, we plot the bound of Theorem \ref{thm:main} with 75\%, 95\%, and 99\% confidence. We consider $m=100,000$ randomized examples and 1,500 guesses of which 1,429 are correct.
        This corresponds to guessing $T_i=+1$ for the largest $750$ scores, $T_i=-1$ for the smallest $750$ scores, and $T_i=0$ for the remaining 98,500 examples. 
    }
    \label{fig:delta-eps}
\end{figure}

Our main contribution is showing that we can audit the differential privacy guarantees of an algorithm with a single run. In contrast, prior methods require hundreds -- if not thousands -- of runs, which is computationally prohibitive for all but the simplest algorithms.
Our experimental results demonstrate that in practical settings our methods are able to give meaningful lower bounds on the privacy parameter $\varepsilon$.

However, while we win on computational efficiency, we lose on tightness of our lower bounds.
We now illustrate the limitations of our approach and discuss the extent to which this is inherent, and what lessons we can learn.

But, first, we illustrate that our method can give tight lower bounds. 
In Figure \ref{fig:guesses-eps-pure}, we consider an idealized setting where the number of guesses changes and the fraction that are correct is fixed at $\frac{e^\varepsilon}{e^\varepsilon+1}$ for $\varepsilon=4$ -- i.e., 98.2\% of guesses are correct.\footnote{The number of correct guesses is rounded down to an integer (which results in the lines being jagged). There are no abstentions.} This is the maximum expected fraction of correct guesses compatible with $(4,0)$-DP.
In this setting the lower bound on $\varepsilon$ does indeed come close to $4$. With 10,000 guesses we get $\varepsilon \ge 3.87$ with 95\% confidence.

Note that the lower bound in Figure \ref{fig:guesses-eps-pure} improves as we increase the number of guesses. This is simply accounting for sampling error -- to get a lower bound with 95\% confidence, we must underestimate to account for the fact that the number of correct guesses may have been inflated by chance. As we get more guesses, the relative size of chance deviations reduces.

\textbf{Limitations:}
Next we consider a different idealized setting -- one that is arguably more realistic -- where our method does not give tight lower bounds. 
Suppose $S_i \in \{-1,+1\}$ indicates whether example $i \in [n]$ is included or excluded.
In Figure \ref{fig:guesses-eps}, we consider Gaussian noise addition. That is, we release a sample from $\mathcal{N}(S_i,4)$. (In contrast, Figure \ref{fig:guesses-eps-pure} considers randomized response on $S_i$.)
Lemma \ref{lem:gauss} gives an upper bound of $(4.38,10^{-5})$-DP.
Unlike for randomized response, abstentions matter here.
We consider 100,000 examples, each of which has a score sampled from $\mathcal{N}(S_i,4)$, where $S_i \in \{-1,+1\}$ is uniformly random.
We pick the largest $r/2$ scores and guess $S_i=+1$. Similarly we guess $S_i=-1$ for the smallest $r/2$ scores. We abstain for the remaining $100,000-r$ examples.
If we make more guesses (i.e., increase $r$), then the accuracy goes down and so does our lower bound. We must trade off between more guesses being less accurate on average and more guesses having smaller relative sampling error.

In Figure \ref{fig:guesses-eps}, the highest value of the lower bound is $\varepsilon \ge 2.675$ for $\delta=10^{-5}$, which is attained by $1439$ correct guesses out of $1510$.
In contrast, the upper bound is $\varepsilon=4.38$ for $\delta=10^{-5}$.
To get a matching upper bound of $\varepsilon \le 2.675$ we would need to set $\delta=0.0039334$. In other words, the gap between the upper and lower bounds is a factor of $393\times$ in $\delta$.

Figure \ref{fig:delta-eps} considers the same idealized setting as Figure \ref{fig:guesses-eps}, but we fix the number of guesses to 1,500 out of 100,000 (of which 1,429 are correct); instead we vary $\delta$ and consider different confidence levels.

\begin{algorithm}[h]
    \caption{Pathological Algorithm}\label{alg:pathological}
    \begin{algorithmic}[1]
        \State \textbf{Input:} $s \in \{-1,+1\}^m$
        \State \textbf{Parameters:} $r \in [m]$, $\varepsilon,\delta \ge 0$, $\beta \in [0,1]$. Assume $0 < m\delta \le r\beta$.
        \State Select $U \subset [m]$ of size $|U|=r$ uniformly at random.
        \State Set $T_i=0$ for all $i \notin U$.
        \State Sample $X \gets \mathsf{Bernoulli}(\beta)$.
        \If{$X=1$}
            \For{$i \in U$}
                \State{Independently sample $T_i \in \{-1,+1\}$ with $\pr{}{T_i=s_i} = \frac{m\delta}{r\beta} + \left(1-\frac{m\delta}{r\beta}\right) \cdot \frac{e^\varepsilon}{e^\varepsilon+1}$.}
            \EndFor
        \ElsIf{$X=0$}
            \For{$i \in U$}
                \State{Independently sample $T_i \in \{-1,+1\}$ with $\pr{}{T_i=s_i} =  \frac{e^\varepsilon}{e^\varepsilon+1}$.}
            \EndFor
        \EndIf
        \State \textbf{Output:} $T \in \{-1,0,+1\}^m$.
    \end{algorithmic}
\end{algorithm}

\textbf{Are these limitations inherent?}
Figures \ref{fig:guesses-eps} \& \ref{fig:delta-eps} illustrate the limitations of our approach. They also hint at the causes: The number of guesses versus abstentions, the $\delta$ parameter, and the confidence all have a large effect on the tightness of our lower bound.

Our theoretical analysis is fairly tight; there is little room to improve Theorem \ref{thm:main}. %
We argue that the inherent problem is a mismatch between ``realistic'' DP algorithms and the ``pathological'' DP algorithms for which our analysis is nearly tight. This mismatch makes our lower bound much more sensitive to $\delta$ than it ``should'' be.

To be concrete about what we consider pathological, consider $M : \{-1,+1\}^m \to \{-1,0,+1\}^m$ defined by Algorithm \ref{alg:pathological}.
This algorithm satisfies $(\varepsilon,\delta)$-DP and makes $r$ guesses with $m-r$ abstentions.
In the $X=1$ case, the expected fraction of correct guesses is $\frac{m\delta}{r\beta} + \left(1-\frac{m\delta}{r\beta}\right) \cdot \frac{e^\varepsilon}{e^\varepsilon+1}$.
This is higher than the average fraction of correct guesses, but if we want confidence $1-\beta$ in our lower bound, we must consider this case, as $X=1$ happens with probability $\beta$.

Intuitively, the contribution from $\delta$ to the fraction of correct guesses should be negligible. However, we see that $\delta$ is multiplied by $m/r\beta$. That is to say, in the settings we consider, $\delta$ is multiplied by a factor on the order of $100\times$ or $1000\times$, which means $\delta=10^{-5}$ makes a non-negligible contribution to the fraction of correct guesses. 

It is tempting to try to circumvent this problem by simply setting $\delta$ to be very small. However, as shown in Figure \ref{fig:delta-eps}, the upper bound on $\varepsilon$ also increases as $\delta \to 0$.

Unfortunately, there is no obvious general way to rule out algorithms that behave like Algorithm \ref{alg:pathological}. The fundamental issue is that the privacy losses of the $m$ examples are not independent; we shouldn't expect them to be independent, but we also shouldn't expect them to be pathologically dependent in reality.

\textbf{Directions for further work:}
Our work highlights several questions for further exploration:
\begin{itemize}

    \item \textbf{Improved attacks:} 
    Our experimental evaluation uses existing attack methods. Any improvements to membership inference attacks could be combined with our results to yield improved privacy auditing.
    
    One limitation of our attacks is that some examples may be ``harder'' than others and the scores we compute do not account for this. When we have many runs, we can account for the hardness of individual examples \cite{carlini2022membership}, but in our setting it is not obvious how to do this.

    \item \textbf{Algorithm-specific analyses:}
    Our methods are generic -- they can be applied to essentially any DP algorithm.
    This is a strength, but there is also the possibility that we could obtain stronger results by exploiting the structure of specific algorithms.
    A natural example of such structure is the iterative nature of DP-SGD.
    That is, we can view one run of DP-SGD as the composition of multiple independent DP algorithms which are run sequentially.

    \item \textbf{Multiple runs \& multiple examples:}
    Our method performs auditing by including or excluding multiple examples in a single training run, while most prior work performs multiple training runs with a single example example included or excluded.
    Can we get the best of both worlds? If we use multiple examples and multiple runs, we should be able to get tighter results with fewer runs.
    
    \item \textbf{Other measures of privacy:}
    Our theoretical analysis is tailored to the standard definition of differential privacy. 
    But there are other definitions of differential privacy such as R\'enyi DP. And, in particular, many of the upper bounds (e.g., Proposition \ref{prop:rdp-dpsgd}) are stated in this language. 
    Hence it would make sense for the lower bounds also to be stated in this language.

    \item \textbf{Beyond lower bounds:} 
    Privacy auditing produces empirical lower bounds on the privacy parameters. In contrast, mathematical analysis produces upper bounds. Both are necessarily conservative, which leaves a large gap between the upper and lower bounds. 
    A natural question is to find some middle ground -- an estimate which is neither a lower nor upper bound, but provides some meaningful estimate of the ``true'' privacy loss.
    However, it is unclear what kind of guarantee such an estimate should satisfy, or what interpretation the estimate should permit.
    
\end{itemize}

\printbibliography

\appendix

\section{Sampling a Fixed-Size Dataset}\label{sec:replace_dp}

Our auditing framework considers randomly including or excluding $m$ examples independently.
This means that the size of the dataset is random. This may be undesirable.

Fortunately, we can fix this, without changing our theoretical analysis. But it does require more examples and it requires changing the definition of DP to consider pairs of datasets differing by the replacement of one person's data, rather than the addition or removal of one person's data (cf.~Remark \ref{rem:addremove}).

Recall that our auditing framework starts with $m$ examples $x_1, \cdots, x_m$ and then samples $S \in \{-1,+1\}^m$ uniformly at random. Then each example $x_i$ is included in the dataset if $S_i=+1$ and excluded if $S_i=-1$. Thus flipping $S_i$ corresponds to adding or removing $x_i$.

Instead we can start with $2m$ examples $x_1, \cdots, x_{2m}$ and then sample $S \in \{-1,+1\}^m$ uniformly. Now, if $S_i=+1$, we include $x_{2i}$ in the dataset and, if $S_i=-1$, we include $x_{2i-1}$ instead. Thus flipping $S_i$ corresponds to replacing $x_{2i}$ with $x_{2i-1}$ or vice versa.

This alternative approach ensures that we always include $m$ out of the $2m$ examples -- i.e., the dataset size is not random.
This still fits the formalism of our theoretical analysis (\S\ref{sec:theory}). However, the DP guarantee of the algorithm being audited (e.g., DP-SGD) must now be with respect to replacement of one example, rather than addition or removal.\footnote{By group privacy, $(\varepsilon,\delta)$-DP for addition or removal implies $(2\varepsilon,(e^\varepsilon+1)\cdot\delta)$-DP for replacement.}
The auditor also needs to change slightly; rather than being given $x_i$ and needing to guess whether or not it is included in the datsets, the auditor is given both $x_{2i}$ and $x_{2i-1}$ and must guess which of the two is included.

\section{Generalization from Differential Privacy}

Our analysis builds on the connection between DP and generalization \cite{dwork2015preserving,dwork2015generalization,bassily2016algorithmic,feldman2017generalization,jung2019new}. We now extend our theoretical results (\S\ref{sec:theory}) to this setting.
The main difference between our analysis in Section \ref{sec:theory} and the prior work on DP and generalization is that we restrict to i.i.d.~binary inputs with a uniform distribution, while prior work considers i.i.d.~inputs from an arbitrary set with an arbitrary distribution. Thus the prior work is more general, but, as we now show, we can reduce the general case to the binary case.

\begin{theorem}[DP implies Generalization]\label{thm:generalization}
    Let $A : \mathcal{X}^n \to \mathcal{Y} \times [0,1]$ be $(\varepsilon,\delta)$-DP (with respect to replacement).
    Let $P$ be a distribution on $\mathcal{X}$.
    Let $q : \mathcal{Y} \times \mathcal{X} \to [0,1]$. For $x \in \mathcal{X}^n$ and $y \in \mathcal{Y}$, denote $q(y,x) = \frac1n \sum_i^n q(y,x_i) \in [0,1]$ and $q(y,P) = \ex{X \gets P}{q(y,X)} \in [0,1]$.
    Then, for all $\gamma\ge \frac32\eta\ge0$, we have \[\pr{X \gets P^n \atop Y \gets A(X)}{q(Y,X)-q(Y,P) \ge \gamma} \le \begin{array}{l} \pr{}{\check{W} \ge \frac{(1+\gamma-\frac32\eta) n}{2}} + 2 \cdot e^{-n\eta^2/2} \\ + \max_{i \in [n]} \frac{2n\delta}{i} \pr{}{\frac{(1+\gamma-\frac32\eta) n}{2} > \check{W} \ge \frac{(1+\gamma-\frac32\eta) n}{2} - i} \end{array},\] where  $\check{W} \gets \mathsf{Binomial}\left( n , \frac{e^\varepsilon}{e^\varepsilon+1} \right)$.
\end{theorem}

The proof of Theorem \ref{thm:generalization} relies on the following technical lemma. This is using what is known as the ``ghost samples'' symmetrization technique \cite[Footnote 2]{steinke2020reasoning}.

\begin{lemma}\label{lem:generalization}
    Let $x^+,x^- \in \mathcal{X}^n$.
    For $s \in \{-1,+1\}^n$, define $x^s \in \mathcal{X}^n$ by $x^s_i = x^+_i$ if $s_i=+1$ and $x^s_i = x^-_i$ if $s_i=-1$.
    Let $A : \mathcal{X}^n \to \mathcal{Y}$ be $(\varepsilon,\delta)$-DP (for replacement).
    Let $q : \mathcal{Y} \times \mathcal{X} \to [0,1]$, and denote $q(y,x) = \frac1n \sum_i^n q(y,x_i) \in [0,1]$ for $y \in \mathcal{Y}$ and $x \in \mathcal{X}^n$.

    Let $S \in \{-1,+1\}$ be uniform.
    Then, for all $v , r \ge 0$,
    \begin{align*}
        &\pr{S \gets \{-1,+1\}^n \atop Y \gets A(x^S)}{q(Y,x^S) - q(Y,x^{-S})\ge \frac{v}{n}} \\
        &~~~~~\le \pr{}{\check{W} \ge \frac{v-r+n}{2}} + \max_{i \in [n]} \frac{2n\delta}{i} \pr{}{ \frac{v-r+n}{2} > \check{W} \ge  \frac{v-r+n}{2}-i}+e^{-r^2/2n},
    \end{align*}
    where $\check{W} \gets \mathsf{Binomial}\left( n , \frac{e^\varepsilon}{e^\varepsilon+1} \right)$.
\end{lemma}
\begin{proof}
    Let $R : [-1,+1] \to \{-1,+1\}$ denote the randomized rounding function. I.e., $\ex{}{R(x)}=x$ for all $x \in [-1,1]$.
    We define $M : \{-1,+1\}^n \to \{-1,+1\}^n$ as follows.
    The inputs $x^+,x^- \in \mathcal{X}^n$ are ``hardcoded'' into $M$ and, for this analysis, we do not consider them private.
    Instead the input is $s \in \{-1,+1\}^n$.
    The algorithm $M(s)$ first runs $A(x^s)$ and then postprocesses the output using the hardcoded information.
    Specifically, given $A(s)=y$, the output $M(s) \in \{-1,+1\}^n$ has a product distribution with $M(s)_i = R(q(y,x^+_i) - q(y,x^-_i)) \in \{-1,+1\}$ for all $y \in \mathcal{Y}$ and all $i \in [n]$.
    That is, for each coordinate $i\in[n]$, we independently randomly round $q(y,x^+_i) - q(y,x^-_i) \in [-1,+1]$ to $\{-1,+1\}$, where $y$ is the output of $A(x^s)$.
    By postprocessing, $M$ is $(\varepsilon,\delta)$-DP. Thus we can apply Theorem \ref{thm:main} to $M$.
    We have, for all $r,v \ge 0$,
    \begin{align*}
        &\pr{S \gets \{-1,+1\}^n \atop Y \gets A(x^S)}{q(Y,x^S) - q(Y,x^{-S})\ge \frac{v}{n}} \\
        &=\pr{S \gets \{-1,+1\}^n \atop Y \gets A(x^S)}{\sum_i^n q(Y,x^S_i) - q(Y,x^{-S}_i) \ge v} \\
        &=\pr{S \gets \{-1,+1\}^n \atop Y \gets A(x^S)}{\sum_i^n (q(Y,x^+_i) - q(Y,x^{-}_i)) \cdot S_i \ge v} \\
        &= \pr{S \gets \{-1,+1\}^n \atop Y \gets A(x^S)}{\ex{R}{\sum_i^n R(q(Y,x^+_i) - q(Y,x^{-}_i)) \cdot S_i} \ge v} \\
        &\le \pr{S \gets \{-1,+1\}^n \atop Y \gets A(x^S),R}{\sum_i^n R(q(Y,x^+_i) - q(Y,x^{-}_i)) \cdot S_i \ge v-r}+e^{-r^2/2n} \tag{Hoeffding \& union}\\
        &=\pr{S \gets \{-1,+1\}^n \atop T \gets M(S)}{\sum_i^n  T_i \cdot S_i  \ge v-r}+e^{-r^2/2n} \\
        &=\pr{S \gets \{-1,+1\}^n \atop T \gets M(S)}{\sum_i^n 2\max\{ 0 , T_i \cdot S_i \}-|T_i| \ge v-r}+e^{-r^2/2n} \tag{$S_i \in \{-1,+1\}$} \\
        &=\pr{S \gets \{-1,+1\}^n \atop T \gets M(S)}{\sum_i^n \max\{ 0 , T_i \cdot S_i \} \ge \frac{v-r+n}{2}}+e^{-r^2/2n} \tag{$|T_i|=1$}\\
        &\le \pr{}{\check{W} \ge \frac{v-r+n}{2}} + \max_{i \in [n]} \frac{2n\delta}{i} \pr{}{ \frac{v-r+n}{2} > \check{W} \ge  \frac{v-r+n}{2}-i}+e^{-r^2/2n},
    \end{align*}
    where $\check{W} \gets \mathsf{Binomial}\left( n , \frac{e^\varepsilon}{e^\varepsilon+1}\right)$.
    Note that Theorem \ref{thm:main} applies with any distribution $\check{W}^*$ satisfying 
    \[\forall v \in \mathbb{R} ~~~ \pr{}{\check{W}^* > v} \ge \sup_{t \in \mathsf{support}(M(s))} \pr{\check{S} \gets \mathsf{Bernoulli}\left(\frac{e^\varepsilon}{e^\varepsilon+1}\right)^n}{\sum_i^n \check{S}_i \cdot |t_i| > v}.\]
    If $t \in \mathsf{support}(M(s))$, then $|t_i| =1 $ for all $i \in [n]$, which implies $\check{W}$ satisfies this requirement.
    In the analysis above, we used Hoeffding's inequality to show that the sum of randomized roundings is close to (within $r$ of) the unrounded sum with high probability and we carry this failure probability $e^{-r^2/2n}$ into the final result.
\end{proof}

\begin{proposition}\label{prop:generalization}
    Let $A : \mathcal{X}^n \to \mathcal{Y}$ be $(\varepsilon,\delta)$-DP (with respect to replacement).
    Let $q : \mathcal{Y} \times \mathcal{X} \to [0,1]$, and denote $q(y,x) = \frac1n \sum_i^n q(y,x_i) \in [0,1]$ for $y \in \mathcal{Y}$ and $x \in \mathcal{X}^n$.
    Let $P$ be a distribution on $\mathcal{X}$.
    Then, for all $\gamma,\eta \ge 0$, we have
    \begin{align*}
        &\pr{X, \widetilde{X} \gets P^n \atop Y \gets A(X)}{q(Y,X)-q(Y,\widetilde{X}) \ge \gamma} \\&~~~\le \pr{}{\check{W} \ge \frac{(1+\gamma-\eta) n}{2}} + \max_{i \in [n]} \frac{2n\delta}{i} \pr{}{\frac{(1+\gamma-\eta) n}{2} > \check{W} \ge \frac{(1+\gamma-\eta) n}{2} - i} + e^{-n\eta^2/2},
    \end{align*}
    where $\check{W} \gets \mathsf{Binomial}\left( n , \frac{e^\varepsilon}{e^\varepsilon+1} \right)$.
\end{proposition}
\begin{proof}
    The proof relies on Lemma \ref{lem:generalization}, which considers $x^+,x^-\in\mathcal{X}^n$ to be fixed. We now average the lemma over these being i.i.d.~samples from $P$, which gives
    \begin{align*}
        &\ex{X^+,X^- \gets P^n}{\pr{S \gets \{-1,+1\}^n \atop Y \gets A(X^S)}{q(Y,X^S) - q(Y,X^{-S})\ge \frac{v}{n}}} \\
        &~~~~~\le \pr{}{\check{W} \ge \frac{v-r+n}{2}} + \max_{i \in [n]} \frac{2n\delta}{i} \pr{}{ \frac{v-r+n}{2} > \check{W} \ge  \frac{v-r+n}{2}-i}+e^{-r^2/2n},
    \end{align*}
    where $\check{W} \gets \mathsf{Binomial}\left( n , \frac{e^\varepsilon}{e^\varepsilon+1} \right)$.
    Since the samples from $P$ are independent, the coordinates of $X^+$ and $X^-$ are interchangeable, so \[\pr{X, \widetilde{X} \gets P^n \atop Y \gets A(X)}{q(Y,X)-q(Y,\widetilde{X}) \ge \gamma} = \ex{X^+,X^- \gets P^n}{\pr{S \gets \{-1,+1\}^n \atop Y \gets A(X^S)}{q(Y,X^S) - q(Y,X^{-S})\ge \frac{v}{n}}}\] for $v = \gamma n\ge 0$. Setting $r=n\eta$ yields the result.
\end{proof}

\begin{proof}[Proof of Theorem \ref{thm:generalization}.]
    Let $X, \widetilde{X} \gets P^n$ be two independent samples.
    Let $Y \gets A(X)$.
    Let $\check{W} \gets \mathsf{Binomial}\left( n , \frac{e^\varepsilon}{e^\varepsilon+1} \right)$.
    By Proposition \ref{prop:generalization}, for all $\gamma, \eta \ge 0$, we have \
    \begin{align*}
        &\pr{X, \widetilde{X} \gets P^n \atop Y \gets A(X)}{q(Y,X)-q(Y,\widetilde{X}) \ge \gamma} \\&~~~\le \pr{}{\check{W} \ge \frac{(1+\gamma-\eta) n}{2}} + \max_{i \in [n]} \frac{2n\delta}{i} \pr{}{\frac{(1+\gamma-\eta) n}{2} > \check{W} \ge \frac{(1+\gamma-\eta) n}{2} - i} + e^{-n\eta^2/2},
    \end{align*}
    By Hoeffding's inequality, for all $\eta \ge 0$, we have \[\forall y \in \mathcal{Y} ~~~~~ \pr{\widetilde{X} \gets P^n}{q(y,\widetilde{X}) - q(y,P) \ge \frac\eta2} \le \exp(-n\eta^2/2).\]
    By a union bound, for all $\gamma \ge \eta \ge 0$, we have \[\pr{X \gets P^n \atop Y \gets A(X)}{q(Y,X)-q(Y,P) \ge \gamma} \le \begin{array}{c} \pr{X, \widetilde{X} \gets P^n \atop Y \gets A(X)}{q(Y,\widetilde{X})-q(Y,P) \ge \gamma-\eta/2} ~~~~~\\~~~~~+ \pr{X, \widetilde{X} \gets P^n \atop Y \gets A(X)}{q(Y,X)-q(Y,\widetilde{X}) \ge \eta/2}\end{array}.\]
    Combining inequalities yields the result:
    \begin{align*}
        &\pr{X \gets P^n \atop Y \gets A(X)}{q(Y,X)-q(Y,P) \ge \gamma} \\
        &~~~\le \pr{}{\check{W} \ge \frac{(1+\gamma-\eta/2-\eta) n}{2}} + 2 \cdot e^{-n\eta^2/2} \\&~~~~~~ + \max_{i \in [n]} \frac{2n\delta}{i} \pr{}{\frac{(1+\gamma-\eta/2-\eta) n}{2} > \check{W} \ge \frac{(1+\gamma-\eta/2-\eta) n}{2} - i}.
    \end{align*}
\end{proof}

\subsection{Comparison to Prior Work on DP \& Generalization}

We now briefly compare our results to the prior work on the connection between DP and generalization \cite{dwork2015preserving,dwork2015generalization,bassily2016algorithmic,feldman2017generalization,jung2019new}. We focus on the work of \citet{jung2019new} as it has the sharpest results in the literature.

Note that the prior work is focused on the setting of adaptive data analysis, while we are focused on the setting of auditing. This difference is mostly cosmetic, but there is a material difference when the prior results are applied to our setting: In addition to outputting guesses, the prior works assume that the algorithm outputs a differentially private estimate of the number of correct guesses. The guarantee then is that this differentially private estimate is close to the distributional average (i.e., only half of the guesses being correct). In contrast, for auditing, we want the true number of correct guesses to be close to the distributional average and don't produce a DP estimate.
We can convert between these two settings using the triangle inequality. 

Below we state the accuracy guarantee that we compare against, followed by a corollary of Theorem \ref{thm:generalization} that applies the triangle inequality and a union bound to ensure that it is directly comparable. 

\begin{theorem}[{\cite[Theorem 3.5]{jung2019new}}]\label{thm:jlnrsms}
    Let $A : \mathcal{X}^n \to \mathcal{Y} \times [0,1]$ be $(\varepsilon,\delta)$-DP (with respect to replacement).
    Let $P$ be a distribution on $\mathcal{X}$.
    Let $q : \mathcal{Y} \times \mathcal{X} \to [0,1]$. For $x \in \mathcal{X}^n$ and $y \in \mathcal{Y}$, denote $q(y,x) = \frac1n \sum_i^n q(y,x_i) \in [0,1]$ and $q(y,P) = \ex{X \gets P}{q(y,X)} \in [0,1]$.
    Suppose \[\pr{X \gets P^n \atop (Y,Z) \gets A(X)}{\left| Z - q(Y,X) \right| \ge \alpha} \le \beta.\]
    Then, for any $c,d>0$, we have 
    \begin{equation}
        \pr{X \gets P^n \atop (Y,Z) \gets A(X)}{\left| Z - q(Y,P) \right| > \alpha + e^\varepsilon-1 + c + 2d} \le \frac{\beta}{c} + \frac{\delta}{d}.\label{eq:jlnrsms}
    \end{equation}
\end{theorem}

\begin{corollary}[Theorem \ref{thm:generalization}, triangle inequality, \& union bound]\label{cor:gen_ada}
    Let $A : \mathcal{X}^n \to \mathcal{Y} \times [0,1]$ be $(\varepsilon,\delta)$-DP (with respect to replacement).
    Let $P$ be a distribution on $\mathcal{X}$.
    Let $q : \mathcal{Y} \times \mathcal{X} \to [0,1]$. For $x \in \mathcal{X}^n$ and $y \in \mathcal{Y}$, denote $q(y,x) = \frac1n \sum_i^n q(y,x_i) \in [0,1]$ and $q(y,P) = \ex{X \gets P}{q(y,X)} \in [0,1]$.
    Suppose \[\pr{X \gets P^n \atop (Y,Z) \gets M(X)}{\left| Z - q(Y,X) \right| \ge \alpha} \le \beta.\]
    Then, for all $\gamma\ge \frac32\eta\ge0$, we have 
    \begin{equation}
        \pr{X \gets P^n \atop (Y,Z) \gets A(X)}{|Z-q(Y,P)| \ge \alpha + \gamma} \le \beta + \begin{array}{l} \pr{}{\check{W} \ge \frac{(1+\gamma-\frac32\eta) n}{2}} + 2 \cdot e^{-n\eta^2/2} \\ + \max_{i \in [n]} \frac{2n\delta}{i} \pr{}{\frac{(1+\gamma-\frac32\eta) n}{2} > \check{W} \ge \frac{(1+\gamma-\frac32\eta) n}{2} - i} \end{array},\label{eq:our-gen}
    \end{equation} where  $\check{W} \gets \mathsf{Binomial}\left( n , \frac{e^\varepsilon}{e^\varepsilon+1} \right)$. 
\end{corollary}

Equations \ref{eq:jlnrsms} and \ref{eq:our-gen} are directly comparable, but it is not immediately obvious how they compare.
By setting $\delta=0$, $\gamma=\frac{e^\varepsilon-1}{e^\varepsilon+1}+c$, $\eta = \frac25c$, and applying Hoeffding's inequality to $\check{W}$, we can simplify Equation \ref{eq:our-gen} to 
\begin{equation}
    \pr{X \gets P^n \atop (Y,Z) \gets A(X)}{|Z-q(Y,P)| \ge \alpha + \frac{e^\varepsilon-1}{e^\varepsilon+1}+c} \le \beta + 3 \cdot e^{-n\frac{2}{25}c^2}
\end{equation}
For comparison, setting $\delta=0$ in Equation \ref{eq:jlnrsms} gives
\begin{equation}
    \pr{X \gets P^n \atop (Y,Z) \gets A(X)}{\left| Z - q(Y,P) \right| > \alpha + e^\varepsilon-1 + c} \le \frac{\beta}{c}.
\end{equation}
Now we can compare the results more easily. 
The $e^\varepsilon-1$ term in the accuracy bound of \citet{jung2019new} is improved to $\frac{e^\varepsilon-1}{e^\varepsilon+1}$ in our result, which is an improvement by a factor of at least two.
This is (arguably) the dominant term, so our result is a significant improvement.
In particular, if $\varepsilon \ge \log 2$, then Equation \ref{eq:jlnrsms} gives a vacuous bound (since the value of $q$ is always in $[0,1]$ anyway), while our bound can be non-vacuous for any value of $\varepsilon$ (as $\frac{e^\varepsilon-1}{e^\varepsilon+1}<1$). 

However, there is another term in the accuracy bound -- i.e., $c$. The failure probability either has a $1/c$ \emph{multiplicative} factor or a $ 3 \cdot e^{-n\frac{2}{25}c^2}$ \emph{additive} factor. How these compare depends on the value of $\beta$.
To give a concrete comparison, suppose $\varepsilon=1/3$, $n=2000$, $\beta=\delta=10^{-5}$, and we want a final failure probability of $0.05$; then Theorem \ref{thm:jlnrsms} gives an error guarantee of $\alpha + 0.397$, while Corollary \ref{cor:gen_ada} gives $\alpha + 0.308$.

\section{Mutual Information Bounds from DP}
Our framework for the theoretical analysis (\S\ref{sec:theory}) is inspired by that of \citet{steinke2020reasoning}.
In this appendix, we use our analysis to also improve one of their results.
Specifically, they show that if $M : \{-1,1\}^m \to \mathcal{Y}$ satisfies $(\varepsilon,\delta)$-DP and $S \in \{-1,1\}^m$ is uniformly random, then \begin{equation}I(S;M(S)) \le (e^\varepsilon-1+\delta) \cdot m \cdot \log e,\end{equation} where $I(\cdot;\cdot)$ denotes the mutual information.\footnote{Throughout this paper we use natural logarithms (so $\log e = 1$), including when defining information-theoretic quantities like mutual information. However, it is common to use base-2 logarithms in information theory (i.e., $\log_2 e \approx 1.44$). To avoid confusion, the statements (outside proofs) in this section are stated in a redundant way so that they would be correct regardless of the base of the logarithm, as long as we are consistent.} 
Prior work \cite{dwork2015generalization,bun2016concentrated} showed that, if $M : \mathcal{X}^m \to \mathcal{Y}$ satisfies $(\varepsilon,0)$-DP and $S \in \mathcal{X}^m$ has as product distribution, then \begin{equation}I(S;M(S)) \le \frac12 \varepsilon^2 \cdot m \cdot \log e .\end{equation}
The latter result is numerically better than the former result, but only holds for pure DP. (The latter result is also not restricted to binary inputs. However, if we do not restrict the input at all, then it is not possible to prove bounds under approximate DP.)

We improve the bound to the following. If $M : \{-1,1\}^m \to \mathcal{Y}$ satisfies $(\varepsilon,\delta)$-DP and $S \in \{-1,1\}^m$ is uniformly random, then \begin{equation}I(S;M(S)) \le \frac18 \varepsilon^2 \cdot m \cdot \log e + \delta \cdot m \cdot \log 2.\end{equation}

\begin{proposition}
    Let $M : \{0,1\}^n \to \mathcal{Y}$ satisfy $(\varepsilon,\delta)$-DP.
    Let $S \in \{0,1\}^n$ be sampled from $\mathsf{Bernoulli}(p)^n$.
    Then
    \begin{align*}
        I(S;M(S)) &\le n\delta h(p) + n(1-\delta) h\left(\frac{p\cdot e^\varepsilon + 1-p}{e^\varepsilon+1}\right) - n(1-\delta)\cdot\left(\log(1+e^{-\varepsilon}) +  \frac{\log(e^\varepsilon)}{e^\varepsilon+1}\right),
    \end{align*}
    where $h(p) := p \log(1/p) + (1-p) \log(1/(1-p))$ is the binary Entropy function.
    
    In particular, if $p=\frac12$, then
    \begin{align*}
        I(S;M(S)) &\le n\delta \log 2 + n(1-\delta)\cdot\left(\log 2 - \log(1+e^{-\varepsilon}) -  \frac{\log(e^\varepsilon)}{e^\varepsilon+1}\right)\\
        &\le n \delta \log 2 + n(1-\delta) \frac{\varepsilon^2}{8} \log e.
    \end{align*}
\end{proposition}
\begin{proof}
    We apply the chain rule and convexity of KL divergence \cite[Lemma 3.7]{feldman2018calibrating}:
    \[I(S;M(S)) = \sum_i^n I(S_i;M(S)|S_{<i}) \le \sum_i^n I(S_i;M(S)|S_{-i}).\]
    Fix $i \in [n]$ and fix $s_{-i} \in \{0,1\}^{n-1}$. Now we must analyze 
    \begin{align*}
        I(S_i;M(S)|S_{-i}=s_{-i}) &= I(S_i;M(S_i,s_{-i})) \\
        &= p \dr{1}{M(1,s_{-i})}{pM(1,s_{-i})+(1-p)M(0,s_{-i})} \\&~~~~~+ (1-p) \dr{1}{M(1,s_{-i})}{pM(1,s_{-i})+(1-p)M(0,s_{-i})}\\
        &= p \dr{1}{Q_1}{Q_p} +(1-p) \dr{1}{Q_0}{Q_p},
    \end{align*}
    where $Q_t := tM(1,s_{-i})+(1-t)M(0,s_{-i})$ for $t\in [0,1]$. 
    
    Since $M$ is $(\varepsilon,\delta)$-DP, we have $Q_t(S) \le e^\varepsilon \cdot Q_{1-t}(S) + \delta$ for all measurable $S \subset \mathcal{Y}$ and $t \in \{0,1\}$.
    Thus we can apply Lemma \ref{lem:kov}: There exist distributions $Q_0',Q_0'',Q_1',Q_1''$ on $\mathcal{Y}$ such that $Q_0 = (1-\delta)\cdot Q_0'+\delta \cdot Q_0''$ and $Q_1 = (1-\delta)\cdot Q_1'+\delta \cdot Q_1''$ and $e^{-\varepsilon} \cdot Q_0'(S) \le Q_1'(S) \le e^\varepsilon \cdot Q_0'(S)$ for all measurable $S \subset \mathcal{Y}$.
    
    Define distributions \[ R_0 := \frac{e^\varepsilon \cdot Q_0' - Q_1'}{e^\varepsilon-1} ~~~~~ \text{ and } ~~~~~ R_1 := \frac{e^\varepsilon \cdot Q_1' - Q_0'}{e^\varepsilon-1}, \] so that $Q_0' = \frac{e^\varepsilon \cdot R_0 + R_1}{e^\varepsilon+1} $ and $Q_1' = \frac{e^\varepsilon \cdot R_1 + R_0}{e^\varepsilon+1} $.
    Hence \[Q_0 = \frac{e^\varepsilon (1-\delta)}{e^\varepsilon +1} \cdot R_0 + \frac{1-\delta}{e^\varepsilon+1} \cdot R_1 + \delta \cdot Q_0''\] and \[Q_1 = \frac{e^\varepsilon (1-\delta)}{e^\varepsilon +1} \cdot R_1 + \frac{1-\delta}{e^\varepsilon+1} \cdot R_0 + \delta \cdot Q_1''.\]
    
    This decomposition (which was first used by \citet{kairouz2015composition}) states that we can view $Q_{s_i} = M(s_i,s_{-i})$ as a postprocessing of an $(\varepsilon,\delta)$-DP randomized response on the bit $s_i$. That is, with probability $\delta$, we output the bit $s_i$ with a flag indicating certainty; with probability $\frac{e^\varepsilon(1-\delta)}{e^\varepsilon+1}$, we output $s_i$ with an uncertain flag; and, with probability $\frac{1-\delta}{e^\varepsilon+1}$, we output $1-s_i$ with the uncertain flag. We can postprocess this to generate a sample from $Q_{s_i} = M(s_i,s_{-i})$ as follows. If we receive $b \in \{0,1\}$ with the uncertain flag, then output a sample from $R_b$. If we receive $b \in \{0,1\}$ with the certain flag, then output a sample from $Q_b''$.
    
    To be formal, define two distributions on the set $[4]=\{1,2,3,4\}$ by
    \begin{align*}
        \widetilde{Q}_0 &= \left( \frac{e^\varepsilon (1-\delta)}{e^\varepsilon+1} , \frac{1-\delta}{e^\varepsilon+1} , \delta , 0 \right), \\
        \widetilde{Q}_1 &= \left( \frac{1-\delta}{e^\varepsilon+1} , \frac{e^\varepsilon(1-\delta)}{e^\varepsilon+1} , 0 , \delta \right).
    \end{align*}
    Define the a randomized postprocessing function $F : [4] \to \mathcal{Y}$ by $F(1)=R_0$, $F(2)=R_1$, $F(3)=Q_0''$, and $F(4)=Q_1''$.
    Then we have $F(\widetilde{Q}_0)=Q_0$ and $F(\widetilde{Q}_1)=Q_1$.
    
    Now we use the postprocessing property (a.k.a.~the data processing inequality):
    \begin{align*}
        I(S_i;M(S)|S_{-i}=s_{-i})
        &= p \dr{1}{Q_1}{Q_p} +(1-p) \dr{1}{Q_0}{Q_p}\\
        &\le p \dr{1}{\widetilde{Q}_1}{\widetilde{Q}_p} +(1-p) \dr{1}{\widetilde{Q}_0}{\widetilde{Q}_p}.
    \end{align*}
    
    A tedious calculation now yields the bound:
    \begin{small}
    \begin{align*}
        & p \dr{1}{\widetilde{Q}_1}{\widetilde{Q}_p} +(1-p) \dr{1}{\widetilde{Q}_0}{\widetilde{Q}_p}\\
        &= p \left( \frac{1-\delta}{e^\varepsilon+1} \log\left(\frac{\frac{1-\delta}{e^\varepsilon+1}}{p \frac{1-\delta}{e^\varepsilon+1} + (1-p) \frac{e^\varepsilon(1-\delta)}{e^\varepsilon+1}}\right)
        +   \frac{e^\varepsilon(1-\delta)}{e^\varepsilon+1} \log\left(\frac{\frac{e^\varepsilon(1-\delta)}{e^\varepsilon+1}}{p \frac{e^\varepsilon(1-\delta)}{e^\varepsilon+1} + (1-p) \frac{1-\delta}{e^\varepsilon+1}}\right)
        +  \delta \log\left(\frac{\delta}{p\delta}\right) \right)\\
        &+ (1-p) \left(
        \frac{e^\varepsilon(1-\delta)}{e^\varepsilon+1} \log\left(\frac{\frac{e^\varepsilon(1-\delta)}{e^\varepsilon+1}}{p \frac{1-\delta}{e^\varepsilon+1} \!+\! (1\!-\!p) \frac{e^\varepsilon(1-\delta)}{e^\varepsilon+1}}\right)
        \!+\!   \frac{1-\delta}{e^\varepsilon+1} \log\left(\frac{\frac{1-\delta}{e^\varepsilon+1}}{p \frac{e^\varepsilon(1-\delta)}{e^\varepsilon+1} \!+\! (1\!-\!p) \frac{1-\delta}{e^\varepsilon+1}}\right)
        \!+\!  \delta \log\left(\frac{\delta}{(1\!-\!p)\delta}\right)  \right) \\        
        &= p \frac{1-\delta}{e^\varepsilon+1} \left(
            \log\left(\frac{1}{p+(1-p)e^\varepsilon}\right) + e^\varepsilon \log\left(\frac{e^\varepsilon}{p e^\varepsilon+(1-p)}\right)
        \right)
        \\&+ (1-p) \frac{1-\delta}{e^\varepsilon+1} \left(
            e^\varepsilon \log\left( \frac{e^\varepsilon}{p + (1-p) e^\varepsilon} \right) + \log \left( \frac{1}{p e^\varepsilon + (1-p)} \right)
        \right)
        \\&+ \delta \left( p \log (1/p) + (1-p) \log(1/(1-p)) \right)\\
        &= \frac{1-\delta}{e^\varepsilon+1} \left(
            (p + (1-p)e^\varepsilon)\log\left(\frac{1}{p+(1-p)e^\varepsilon}\right) + (1-p)e^\varepsilon \varepsilon
            + (p e^\varepsilon + 1-p)\log\left(\frac{1}{p e^\varepsilon + 1-p}\right) + p e^\varepsilon \varepsilon
        \right)
        \\&+ \delta h(p)\\
        &= (1-\delta) \left(
            \frac{p + (1-p) e^\varepsilon}{e^\varepsilon + 1} \log \left( \frac{e^\varepsilon+1}{p + (1-p) e^\varepsilon}\right) 
            + \frac{p e^\varepsilon + 1-p}{e^\varepsilon + 1} \log \left( \frac{e^\varepsilon+1}{p e^\varepsilon + 1-p}\right)
             - \log(e^\varepsilon + 1) + \frac{e^\varepsilon \varepsilon}{e^\varepsilon + 1}
        \right) \\&+ \delta h(p)\\
        &= (1-\delta) \left( h \left( \frac{p+(1-p)e^\varepsilon}{e^\varepsilon+1} \right) - \log(e^\varepsilon + 1) + \frac{e^\varepsilon \varepsilon}{e^\varepsilon + 1} \right)+ \delta h(p) \\
        &= (1-\delta) \left( h \left( \frac{p e^\varepsilon +(1-p)}{e^\varepsilon+1} \right) - \log(1+e^{-\varepsilon}) - \frac{\varepsilon}{e^\varepsilon + 1} \right)+ \delta h(p) .
    \end{align*}
    \end{small}
    Combining inequalities and summing over $i \in [n]$ yields the first part of the result. 
    The final part of the result is the bound \[\forall \varepsilon \ge 0 ~~~~~ g(\varepsilon) := \log 2 - \log(1+e^{-\varepsilon}) - \frac{\varepsilon}{e^\varepsilon+1} \le \frac{\varepsilon^2}{8},\]
    which can be verified by showing that $g(0)=g'(0)=0$ and $\forall \varepsilon \ge 0 ~~ g''(\varepsilon) \le \frac14$ (or by \href{https://www.google.com/search?q=y\%3Dln\%282\%29-ln\%281\%2Bexp\%28-x\%29\%29-x\%2F\%281\%2Bexp\%28x\%29\%29\%2Cy\%3Dx\%5E2\%2F8}{\color{black}{plotting it}}).
\end{proof}

\section{Implementation of Theorem \ref{thm:main}}\label{app:code}
On the next page is \texttt{Python} pseudocode implementing Corollary \ref{cor:ternary}.
Some example usage:
\begin{itemize}
    \item Suppose the auditor correctly guesses $v=75$ out of $m=r=100$ examples, with no abstentions. We have $\frac{75}{100} = \frac34 = \frac{e^{\log 3}}{e^{\log 3} +1}$. So we would expect this to correspond roughly to $\varepsilon=\log 3 \approx 1.09$. Theorem \ref{thm:main} gives p-value of $0.553$ for the null hypothesis $\varepsilon \le \log 3$ and $\delta=0$; to obtain this result call \texttt{p\_value\_DP\_audit(100,100,75,math.log(3),0)} in the code below.
    If we want 95\% confidence, we obtain the lower bound $\varepsilon \ge 0.702$ by calling \texttt{get\_eps\_audit(100,100,75,0,0.05)}.
    If we set $\delta=10^{-4}$, we obtain the weaker lower bound $\varepsilon \ge 0.699$ by calling \texttt{get\_eps\_audit(100,100,75,1e-4,0.05)}.
    
    \item Suppose the auditor correctly guesses $v=75$ out of $r=100$ guesses, but with a total of $m=1000$ examples. I.e., the auditor abstains on $m-r=900$ examples.
    We obtain a lower bound of $\varepsilon \ge 0.673$ for $\delta=10^{-4}$ and 95\% confidence. (This is slightly weaker than the $\varepsilon \ge 0.699$ lower bound we get when there are no abstentions.) This is obtained by calling \texttt{get\_eps\_audit(1000,100,75,1e-4,0.05)}.

\end{itemize}

\begin{center}
\framebox{
\begin{minipage}{0.85\textwidth}
\begin{footnotesize}
\begin{verbatim}
# m = number of examples, each included independently with probability 0.5
# r = number of guesses (i.e. excluding abstentions)
# v = number of correct guesses by auditor
# eps,delta = DP guarantee of null hypothesis
# output: p-value = probability of >=v correct guesses under null hypothesis
def p_value_DP_audit(m, r, v, eps, delta):
  assert 0 <= v <= r <= m
  assert eps >= 0
  assert 0 <= delta <= 1
  q = 1/(1+math.exp(-eps))  # accuracy of eps-DP randomized response
  beta = scipy.stats.binom.sf(v-1, r, q)  # = P[Binomial(r, q) >= v]
  alpha = 0
  sum = 0  # = P[v > Binomial(r, q) >= v - i]
  for i in range(1, v + 1):
      sum = sum + scipy.stats.binom.pmf(v - i, r, q)
      if sum > i * alpha:
        alpha = sum / i
  p = beta + alpha * delta * 2 * m
  return min(p, 1)

# m = number of examples, each included independently with probability 0.5
# r = number of guesses (i.e. excluding abstentions)
# v = number of correct guesses by auditor
# p = 1-confidence e.g. p=0.05 corresponds to 95%
# output: lower bound on eps i.e. algorithm is not (eps,delta)-DP
def get_eps_audit(m, r, v, delta, p):
  assert 0 <= v <= r <= m
  assert 0 <= delta <= 1
  assert 0 < p < 1
  eps_min = 0  # maintain p_value_DP(eps_min) < p
  eps_max = 1  # maintain p_value_DP(eps_max) >= p
  while p_value_DP_audit(m, r, v, eps_max, delta) < p: eps_max = eps_max + 1
  for _ in range(30):  # binary search
    eps = (eps_min + eps_max) / 2
    if p_value_DP_audit(m, r, v, eps, delta) < p:
      eps_min = eps
    else:
      eps_max = eps
  return eps_min
\end{verbatim}
\end{footnotesize}
\end{minipage}
}
\end{center}

\end{document}